\documentclass[a4paper,twoside,11pt,reqno]{amsart}
\usepackage[colorlinks=true, linkcolor=blue, citecolor=blue, urlcolor=blue]{hyperref} % hyperlink 
\usepackage{pdfsync}
\usepackage{ulem}

\usepackage{graphicx}
\usepackage{geometry}
\usepackage{color} 
\usepackage{enumerate}

\usepackage{framed}
\definecolor{shadecolor}{rgb}{0.9,0.9,0.9}

\usepackage{amsmath}
\usepackage{amsfonts}
\usepackage{amssymb}
\usepackage{amsthm} 
\usepackage{amsbsy} % bold face for math symbols 
\usepackage{manfnt} % bourbaki Z
\usepackage{bbm}  %  bbm fonts 
\usepackage{amscd} % diagram
\usepackage{mathabx}

\usepackage{caption}
\usepackage{subcaption}
\usepackage{float} 
\usepackage{bbm} 
\usepackage{cleveref}
\theoremstyle{plain}% default 
\newtheorem{thm}{Theorem}[section]
\newtheorem{lem}[thm]{Lemma}
\newtheorem{prop}[thm]{Proposition}

\theoremstyle{definition}

\newtheorem{exa}{Example}
\crefname{exa}{Example}{Examples}

\theoremstyle{remark} 
\newtheorem{rmk}{Remark} 

\newcommand{\new}{\newcommand}

\providecommand{\nor}[1]{\lVert{#1}\rVert}

\providecommand{\abs}[1]{\lvert{#1}\rvert}\providecommand{\set}[1]{\{#1\}}
\providecommand{\scal}[2]{\langle{#1},{#2}\rangle}

\providecommand{\tr}[1]{\operatorname{Tr}{\left(#1\right)}}

\providecommand{\argmin}[2]{\operatornamewithlimits{argmin}\limits_{#1}{#2}}
\providecommand{\wh}[1]{\widehat{#1}}
\providecommand{\wc}[1]{\widecheck{#1}}
\providecommand{\Id}{\operatorname{I}}

\new{\sinc}{\operatorname{sinc}}

\new{\R}{\mathbb R}
\new{\C}{\mathbb C}
\new{\N}{\mathbb N}
\new{\Z}{\mathbb Z}

\new{\hh}{\mathcal H}
\new{\kk}{\mathcal K}

\new{\F}{\mathcal F}

\new{\la}{\lambda}
\new{\eps}{\epsilon}
\new{\ga}{\gamma}

\new{\conv}{\convolution}
\new{\X}{L^1}
\new{\Y}{L^2}

\providecommand{\PP}[1]{ \mathbb P\left[#1\right]}
\providecommand{\EE}[1]{ \mathbb E\left[#1\right]}

\new{\wn}{w_n}
\new{\Cn}{ C_n}
\new{\Dn}{ {\mathbf Z}}

\new{\Mh}{\wh{A}}
\new{\M}{\mathcal M_{\wh{K}}}

\new{\Sr}{S}
\new{\Sn}{S_n}
\new{\Tn}{\Sigma_n}

\usepackage{cleveref}
\usepackage[numbers]{natbib}
\begin{document}
% front matter
\title{Learning  convolution operators on compact abelian groups}

\author{Emilia Magnani}
\address{Emilia Magnani, 
T\"ubingen AI Center, University of T\"ubingen, T\"ubingen, Germany}
\email{emilia.magnani@uni-tuebingen.de }

\author{Ernesto De Vito}
\address{E.~De Vito, MaLGa,, DIMA, Universit\`a degli Studi di Genova,  Via
  Dodecaneso 35, Genova,   Italy  }
\email{ernesto.devito@unige.it}

\author{Philipp ~Hennig}
\address{Philipp Hennig, T\"ubingen AI Center, University of T\"ubingen, T\"ubingen, Germany}
\email{ philipp.hennig@uni-tuebingen.de }

\author{Lorenzo~Rosasco}
\address{L. Rosasco, MaLGa, DIBRIS,  Universit\`a degli Studi di Genova,  Via
  Dodecaneso 35, Genova, \\
  Center for Brains Minds and Machine, MIT, Cambridge USA\\
  Istituto Italiano di Tecnologia, Genova, Italy.  }
\email{lrosasco@mit.edu}

\date{\today} 

\begin{abstract}
We consider the  problem of learning convolution operators associated to  compact Abelian groups.
We study a regularization-based approach and provide corresponding learning guarantees under natural regularity conditions on the convolution kernel. 
More precisely, we assume the convolution kernel  is  a function in a translation invariant 
Hilbert space and analyze a natural  ridge regression (RR) estimator. 
Building on existing results for RR, we  characterize the accuracy of the estimator in terms of finite sample bounds. 
Interestingly,  regularity assumptions which are classical in the analysis of RR,  have a novel and natural interpretation  in terms of space/frequency localization. 
Theoretical results are illustrated by  numerical simulations.   
\end{abstract}



%\tableofcontents

\maketitle

\section{Introduction}\label{sec:intro}
The key problem in machine learning is estimating an input/output function \( f \) of interest from random input/output pairs \( (x_i, y_i)_{i=1}^n \). Classically, inputs are vectors in \( \mathbb{R}^d \), and outputs are binary or scalar values. However, as the scope and number of machine learning applications expand frantically, it is interesting to consider functional relationships between more general inputs and outputs. Relevant to this study is the case where both inputs and outputs are elements of infinite-dimensional spaces, such as Hilbert or Banach spaces, so that \( f \) can be viewed as an operator. This setting has recently received considerable attention, driven by applications in image and signal processing, and more generally in scientific and engineering contexts where data are described by integral and partial differential equations (PDEs), see \cite{kovachki24} and references therein. 

In this paper, we focus on a special class of linear operators, namely convolution operators on an Abelian group $G$,
\[
x \conv w_*(t) = \int_G x(\tau) w_*(t - \tau) \, d\tau.
\]
We consider a statistical framework where the inputs are random signals, and the outputs are noisy images of the convolution operator applied to the inputs, expressed as
\begin{equation}\label{firsteq}
y_i = x_i \conv w_* + \epsilon_i.
\end{equation}
We address the case where translations, and hence convolutions, are defined by a compact Abelian group. Convolution operators form a special class of linear operators that can be characterized by the convolution kernel \( w_* \), together with the properties of the underlying group. As we show, tools from harmonic and Fourier analysis can be employed to gain insights into the structure of the problem and develop a tailored analysis. Before describing our main contributions, we provide some context for our study, discussing a number of related works and results.

\subsection{Related work}\label{sec:related-work}
While we are not aware of studies focusing specifically on learning guarantees for convolution operators, this problem is related to different yet related questions. We will briefly review these connections next.\\
\textbf{Operator learning.} As already mentioned, operator learning has received significant attention lately. On the one hand, there has been a growing literature on neural network approaches, such as the so-called neural operators \cite{kovachki23}; see also \cite{goswami23} and the references therein for an overview. On the other hand, learning-theoretic studies have largely focused on learning linear operators with kernel methods, since they are amenable to a more complete analysis. A recent survey can be found in \cite{kovachki24}, whereas a partial list of references includes \citep{de2023convergence,de2005learning,Trabs_2018,mollenhauer2022learning,JMLR:v17:11-315,arridge2019solving,Bleyer_2013,tabaghi2019learning}. These latter studies consider estimators and technical tools analogous to those in this paper. Indeed, the convolution operators we consider are an example of linear operators. However, their special structure allows for a tailored analysis that highlights the specific structure of the problem. In particular, it is possible to focus on the convolution kernel as the primary object of interest, using functional analytic tools rather than operator analytic tools, and learning bounds in different norms can be derived, leveraging the regularity properties of the convolution kernel. 
A recent contribution in this direction is \cite{zhang2025minimax}, which considers a more general class of integral operators and derives minimax rates in certain spectral Sobolev spaces. 
While their approach covers many integral kernels, our setting leverages the commutative group structure and frequency-domain localization, leading to interpretations that are not considered in their framework.
In particular, regularity assumptions needed to derive learning bounds in our setting can be related to localization properties.\\
\textbf{Learning Green functions.}
Motivated by applications in PDEs, another class of linear operators that has been considered are integral kernel operators, where the kernel can be related to Green functions; see  \cite{JMLR:v23:22-0433} and references therein. 
In our setting, the hypothesis space has a special structure, since it is defined by a translational invariant kernel and reflects the group structure. The analysis in this case cannot be recovered from the case of general integral kernel operators. \\
\textbf{Functional regression.}
The problem we study is related to functional data analysis, where input/output data are represented as functions in a continuous domain; see, e.g., \cite{ramsay,Mas2009LinearPF} and \cite{10.3150/12-BEJ469,hormann2015note,REIMHERR201562}. For example, the studies in \cite{lian2015minimax,kupresanin2010rkhs,10.1214/09-AOS772} address the analysis within Reproducing Kernel Hilbert spaces. Our framework is more general because we consider the general Abelian group setting. At the same time, it is more specific since we focus on convolution kernels, and we can derive and interpret results in different norms.\\
\textbf{Linear time-invariant system estimation.}
Learning convolution kernels and convolution operators is close to identification of linear time-invariants in control theory and signal processing. We refer to \cite{ljung1998system} for a classic reference and to  \cite{brunton2022modern} for modern data-driven approaches that use optimization and machine learning. Estimating linear time-invariant systems involves learning the so-called impulse response or transfer function, which can often be represented as a convolution operator. Indeed, the tools we employ, such as commutative Harmonic Analysis, parallel the techniques used in identification of linear time-invariant systems. Our contributions extend these ideas by focusing on a more general framework where the underlying structure is defined by a compact Abelian group. Moreover,  we consider random inputs and additive noise in the observations, making the problem inherently statistical. \\
\textbf{Blind deconvolution.}
The problem of learning a convolution kernel is related to the so-called blind deconvolution problem; see, e.g., \citep{justen2009general}. A number of machine learning approaches primarily based on neural networks have been considered for this problem \cite{schuler2015learning,6618928,EGMONTPETERSEN20022279,de2003nonlinear,burger2001regularization,Kundur489268}, particularly with respect to the problem of image recovery from blurred photographs \cite{motion10.1145/1618452.1618491,Li978-3-642-15549-9_12}. Typically, only a discrete setting is considered, and no theoretical results are developed. Closely related to our study is the approach in \cite{schuler2015learning}, which focused on algorithmic aspects. Here, we complement these latter results by considering a continuous setting and deriving learning theoretic guarantees.

\subsection{Contribution}
In the context of the learning model in \cref{firsteq}, our main contribution is the analysis of the learning properties of a  ridge regression estimator in terms of non-asymptotic learning bounds. More precisely, we consider convolution operators defined by a compact Abelian group \( G \). With this choice, we identify the input and output spaces with \( L^1(G) \) and \( L^2(G) \), respectively, which are the Lebesgue spaces defined by the Haar measure associated with \( G \). We further assume that the convolution kernel $w^*$ in \cref{firsteq} belongs to  \( L^2(G) \), so that the convolution operator is well defined from $L^1(G)$ into $L^2(G)$.

Moreover, to characterize and leverage the potential regularity properties of the convolution kernel, we consider a translation-invariant  Hilbert space $\hh$ as the hypothesis space in which the estimator of $w^*$ is sought. Under the assumption that $\hh\subset L^2(G)$, an associated  ridge regression estimator is then studied. Both the properties of translation-invariant  hypotheses space $\hh$ and the computations required by the ridge regression estimator can be given a special characterization using the Fourier transform associated with the group. The learning error is studied by adapting results from ridge regression theory. However, we do not assume that the hypothesis space is a reproducing kernel Hilbert space, as is usual for kernel methods in machine learning, but we exploit the fact that $\hh$ is translation invariant. 

In particular, we consider the error decomposition introduced in \cite{caponnetto2007optimal} and the refinements developed in \cite{rucaro15}. Although convolution operators can be studied as a special case of linear operator learning, a tailored analysis is instructive and highlights the special structure of the problem. Specifically, the operator norms and functional norms of the convolution kernel can be related, and different functional norms can be considered. Furthermore, standard regularity assumptions (source and capacity conditions \cite{rosasco2008model,caponnetto2007optimal}) can be described in terms of space/frequency localization properties of the signals.
In particular, our results show that the input signal localization affects estimation differently depending on the norm used to measure the learning error. Finally, we illustrate some of the theoretical findings with numerical simulations.

\subsection*{Plan of the paper}
The remainder of the paper is organized as follows. In Section~2, we provide essential background on group theory and Harmonic Analysis. Section~3 introduces our statistical model and methodology for learning convolution operators, including the use of translation-invariant  Hilbert spaces and  ridge regression. In Section~4, we present the theoretical analysis of the learning error, discuss the main results, and explore the implications of a-priori assumptions in terms of space and frequency localization properties. Section~5 is devoted to numerical simulations, where we validate the theoretical error bounds and demonstrate the applicability of our framework to approximating heat kernels in the context of partial differential equations. We conclude the paper in Section~6 with a summary of our contributions and directions for future research. Detailed proofs of our results are provided in the appendices.
\subsection*{Notation}
If $v,w$ are vectors in $\R^d$, $v \cdot w$ denotes the scalar product
between $v$ and $w$, and $\abs{v}$ is the Euclidean norm of $v$. If
$A$ is a bounded operator between two Banach spaces $E$ and $F$, we
denote by $\nor{A}_{E,F}$ the operator norm, by $A^*: F^*\to A^*$ the
adjoint and by $\mathcal B(E,F)$ the Banach space of bounded linear
operators between $E$ and $F$ endowed with operator norm. 
If  $A$ is a self-adjoint operator on a Hilbert space, we denote by $\sigma(A)$ the spectrum of $A$.

\section{Background on group theory  and Harmonic Analysis}\label{sec:setup}

In this section,  we recall basic notions and facts from  group theory and
from commutative Harmonic Analysis. In particular, we  focus on compact Abelian groups and introduce the  notions of convolution operators and Fourier transform, which we illustrate with some examples. We refer to  some  standard reference such as \cite{rudin82} for further readings. 

Let $G$ be a compact Abelian group, we denote
by $+$ the (additive) group law  and by $t$ the elements of $G$. 
We let $dt$ be the Haar measure  on $G$ normalized to $1$.  Given $p\in [1,+\infty]$, we
let   $L^p=L^p(G,dt)$ be the corresponding  Lebesgue space with norm
$\nor{\cdot}_{p}$ and we denote by $\scal{\cdot}{\cdot}_2$ the scalar product in
$L^2$. 

The dual group of $G$ is denoted by $\wh{G}$, and is a discrete 
Abelian group. The elements of $\wh{G}$ are denoted by $\xi$,
and, for each $\xi\in \wh{G}$ 
\[
G\ni t\mapsto \scal{\xi}{t} \in \C
\] 
is the corresponding character. Since $\wh{G}$ is discrete, the Haar
measure of $\wh{G}$ is the counting measure and the corresponding
Lebesgue spaces are  $\ell^p=\ell^p(\wh{G})$. 

Let $\F:L^1\to \ell^\infty $ be the Fourier transform 
\[
(\F x)_{\xi} = \int_{G} x(t) \, \overline{\scal{\xi}{t}} \, dt,  \qquad \xi\in\wh{G},
\]
where $x\in L^1$. With a standard slight abuse of notation, we let
$\F^{-1}:\ell^1\to L^\infty $ be defined as
\[
(\F^{-1} \wh{x})(t) =  \sum_{\xi\in \wh{G}} \wh{x}_{\xi} \,\scal{\xi}{t},
\qquad t\in G, 
\]
where $\wh{x}=(\wh{x})_{\xi\in\wh{G}}\in \ell^1$. The Fourier transform $\F$, restricted to
$L^2\subset L^1$, is a unitary map  onto $\ell^2\subset\ell^\infty$
and its inverse is given by the unitary extension of 
$\F^{-1}$ from $\ell^1$ to $\ell^2$ and  it is equal to $\F^*$. 

For any $x\in L^p$, the function  $\wc{x}\in L^p$ is defined as
\[\wc{x}(t)=\overline{x(-t)},\]
and, for any $x\in L^1$, 
\begin{equation}
  \label{eq:3}
  \F \wc{x} = \overline{\F x}.
\end{equation}
Moreover, for any $t\in G$ the translation operator is defined as
\[
T_t: L^1 \to L^1 \qquad (T_t x)(s)=x(s-t),
\]
which is a surjective isometry and it holds true that
  \begin{equation}
(\F T_t x)_{\xi} = \overline{\scal{\xi}{t}} (\F x)_{\xi}.\label{eq:94}
\end{equation}
Given $x\in L^1$ and $y\in L^p$, we set $x\conv y $ be the convolution 
\[
x\conv y (t) = \int_{G} x(\tau) y(t-\tau)\, d\tau,  \quad t \in G,
\]
so that $x\conv y\in L^p$,
\begin{equation}
  \label{eq:1}
  \nor{x\conv y}_p\leq \nor{x}_1 \nor{y}_p\ 
\end{equation}
which implies that  the bilinear map 
\begin{equation}
  \label{eq:96}
  L^1\times L^p\ni (x,y)\mapsto x\conv y \in L^p
\end{equation}
is jointly continuous. 
For any fixed $y\in L^p$, we define the convolution operator
\[
C_y : L^1 \to L^p, \quad (C_y x)(t) = x\conv y(t), \quad t\in G.
\]
The inequality \eqref{eq:1} implies that $C_y$ is a bounded linear operator.

If $x\in L^1$, $y\in L^p$ and $z\in L^q$ with  $1/p+1/q=1$, 
\begin{equation}
  \label{eq:33}
  \scal{x\conv y}{z}_{p,q}=\scal{y}{\wc{x}\conv z}_{p,q} ,
\end{equation}
where $\scal{\cdot}{\cdot}_{p,q}$ is the sequilinear duality pairing between $L^p$
and $L^q$.  The convolution theorem states that
\begin{equation}
  \label{eq:2}
  (\F (x\conv y))_\xi = (\F x)_\xi (\F y)_\xi
\end{equation}
for any $x,y\in L^1$. 

Let   $\set{e_\xi}_{\xi\in \wh{G}}$ be the Fourier base of $L^2$ and
$\set{\wh{e}_\xi}_{\xi\in \wh{G}}$ be the canonical base of $\ell^2$,
 i.e.\ for all $\xi\in\wh{G}$ 
\begin{alignat*}{1}
  e_\xi(t) = \scal{\xi}{t},  \qquad t\in G, \qquad\qquad    (\wh{e}_\xi)_{\xi'}  =\delta_{\xi,\xi'}, \qquad \xi'\in\wh{G},
\end{alignat*}
then
\begin{equation}
  \label{eq:73}
  \begin{split}
    & \F e_\xi  = \wh{e}_\xi,  \qquad \qquad \xi\in\wh{G} \\
 & x \conv e_\xi  = (\F x)_\xi \ e_\xi,  \qquad \xi\in\wh{G} \\
& \scal{y}{e_\xi}_2  =(\F y)_\xi,  \qquad y\in L^2, \ \xi\in\wh{G}.
  \end{split}\quad 
\end{equation}
The primary examples of the group $G$ are the $d$-dimensional torus and
the group of circulant matrices. 
\begin{exa}[Torus]
Fix $d\in \N$ with $d\ge 1$. Let $G=(\R/\Z)^d\simeq [0,1]^d$,   regarded as the additive group. The Haar measure $dt$ is the {\it restriction} of
the Lebesgue measure to $[0,1]^d$, the dual group $\wh{G}$ is $\Z^d$ with the pairing 
\[
\scal{\xi_\ell}{t} = e^{2\pi i\ell\cdot t} \qquad  t\in (\R/\Z)^d\quad
\ell\in\Z^d,
\]
where, for sake of clarity, we denote  $\xi=\ell\in\Z^d$,  and the Haar measure  of $\wh{G}$ is the counting measure on $\Z^d$.

\end{exa}
\begin{exa}[Circulant matrices] \label{example:circulant}
Fix $N\in \N$ with $N\geq 2$, and set $G=\Z_N=\set{0,\ldots, N-1}$ regarded
as an additive group modulo $N$.  We denote the elements of $G$ by
$j=0, \ldots, N-1$.  The Haar measure is the counting measure, the dual group
 $\wh{G}$ coincides with $\Z_N$ with the pairing 
\[
\scal{\xi_j}{j'} = e^{i\frac{ j j'}{N}} \qquad   j,j'\in \Z_N\ .
\]
The convolution is given by
\begin{equation}
\label{discrete_conv}
  ( {x} \ast {y} )_j= \sum _{j'=0}^{N-1} {x}_{j'} {y}_{j-j'}\ .
\end{equation}
The indexing ${x}= ({x}_0, \dots , {x}_{n-1})$ is possible
because the indices are evaluated mod $n$. As vector spaces, 
$L^P=\ell^p=\C^N$,  the convolution operator 
$C_y: \C^N \to \C^N, \ C_y x= x\conv y$ 
is the $N\times N$ circulant matrix
\begin{equation} \label{circulant}
C_y = 
\begin{bmatrix} 
    {y}_{0} & {y}_{N-1} & \dots & \dots & {y}_1 \\
    {y}_1 & {y}_0  & & & \vdots & \\
    {y}_2 & {y}_1  & \ddots & & \vdots & \\
    \vdots & \vdots  & &  & \vdots \\
    {y}_{N-1} &  {y}_{N-2}  & & \dots    & {y}_{0} 
    \end{bmatrix}.
\end{equation}
\end{exa}
Provided with the above discussion we next describe the problem of learning convolution operators from random samples using a  ridge regression approach.

\section{Learning convolution operators with regularization}
In this section, we provide a statistical learning framework to learn  convolution operators. Then, we 
introduce translation invariant Hilbert spaces and the   ridge regression estimator we study. 

\subsection{Statistical model}\label{sec:stat_model}
We let $\mathcal X=L^1$ and $\mathcal Y=L^2$ and we  view $\X$ as
the input space and $\Y$ as the output space.  We let $\hh \subseteq
L^2$ be a Hilbert space of hypothesis and $j:\hh \hookrightarrow L^2$ the canonical embedding. 
Let $(X,Y)$ be a pair of random variables such that 
\begin{enumerate}[i)]
 \item the random variable $X$ takes values in $\X$ and it is bounded,
    i.e.\ 
\begin{equation}
  \label{eq:17}
  \nor{X}_1\leq D_X,  \qquad \text{almost surely}
\end{equation}
for some constant $D_X>0$; 
\item the random variable $Y$ takes values in $\Y$ and satisfies
\begin{equation}
  \label{eq:14}
  Y= C_* X + \eps, 
\end{equation}
where $C_*:= C_{w_*}$, with $ C_\ast x = x\conv {j(w_*)}$ for some $w_*\in\hh$, and 
and  $\eps$ is a random variable in $\Y$ such that
\begin{equation}
  \label{eq:16}
  \EE{\eps\mid X}=0 \qquad \EE{\nor{\eps}_{L^2}^m\mid X }\leq
  \frac{m!}{2} M_\eps^{m-2} \sigma_\eps^2 \quad m\geq 2
\end{equation}
for some $M_\eps,\sigma_\eps>0$. 
\end{enumerate}
Eqs.~\eqref{eq:17}--\eqref{eq:16} describe a natural regression model: the random output $Y$ is a
noisy image of the random input $X$ convolved with an unknown convolution kernel $w_*$ belonging to $\hh$. This last
assumption corresponds to the
assumption that the model is  well-specified in the context of regression.  

We will see that deriving learning bounds requires assumptions on the convolution kernel \( w_* \) and the distribution on the space of input signals \( \mathcal{X} \). These assumptions are standard in the theory of ridge regression, but can be given a more explicit interpretation in the context of convolution operator learning, particularly in terms of the localization properties of input signals. The following two examples illustrate how  the input variable \( X \) can be well-localized either in the frequency domain or the spatial domain.

\begin{exa}[Frequency localization]\label{freq_loc}
Let $G=\R/\Z=[0,1]$ be the one-dimensional torus, let $
(p_\ell)_{\ell\in\Z}$ be a probability distribution on $\Z$ and let
$X\in L^1$  be such that 
\begin{equation}    \label{rv-ex-freq-loc}  
\PP{ X= e_\ell}= p_\ell \qquad \ell\in\Z,
\end{equation} 
where for each $\ell\in\Z$, the function  $e_\ell(t)= e^{2\pi i \ell t}$ is the trigonometric
monomial. Clearly, $X$ is localized in the frequency domain at  each point
$\ell$ with probability $p_\ell$, whereas in the space domain $X$ is not
localized, since  $\abs{X(t)}=1$ for all $t\in [0,1]$ with probability
$1$. 
\end{exa}

\begin{exa}[space localization]\label{space_loc}
Let $G=\R/\Z=[0,1]$ be the one-dimensional torus and $\tau$ be a random
variable taking value in $G$.  Fix $0<\delta\leq 1/2$ and define 
\begin{equation}    \label{rv-ex-space-loc}         
 X=\frac{1}{2\delta}  T_\tau \mathbbm{1}_{[-\delta, \delta]}\qquad
 X(t) =
 \begin{cases}
   \frac{1}{2\delta} &   \tau-\delta\leq  t \leq \tau+\delta \\
     0 & \text{otherwise}
 \end{cases}
\end{equation}
so that, for small $\delta$, $X$ is localized in a
$\delta$-neighbourhood of the random point $\tau$.  Since the Fourier
coefficients of $X$ are given by 
 \begin{equation}
\wh{X}(t)_\ell = \frac{1}{2\delta} \int_{-\frac 12}^{\frac 12} 
\mathbbm{1}_{[-\delta, \delta]}(t)
e^{-2 \pi i \ell (t+\tau)}  \ dt  
 = \frac{e^{-2 \pi i \ell \tau}}{2\delta}
\int_{-\delta}^{\delta}   e^{-2 \pi i \ell t} \ dt 
= e^{-2 \pi i \ell \tau } \text{sinc} ({ 2\pi \delta \ell } ),
\end{equation}
then $\abs{\wh{X}(t)_\ell}=\text{sinc} (2\pi
  \delta \ell )\simeq 1$  for all $|\ell|\ll \frac{1}{\delta}$ with
probability $1$, so that for small $\delta$,  $X$ is not localized in
the frequency domain. 
\end{exa}

\subsection{Translation invariant Hilbert spaces}
We now describe the class of hypothesis spaces we consider. Given the properties of convolution operators, a natural choice is  Hilbert spaces invariant under translations by any elements $t\in G$. These spaces can be easily characterized using the Fourier transform, as we recall next. We refer to \cite{berg84},  for further details. Indeed, let $\set{\wh{K}}_{\xi\in \wh{G}}$ be a family   such that 
\begin{equation}
  \label{eq:31}
 0\leq \wh{K}_{\xi} \leq D^2_{K}, \qquad \xi\in \wh{G},
\end{equation}
for some $D_K>0$.  Define the space 
  \begin{equation}
\hh = \set{ w\in L^2\mid \sum_{\xi\in \wh{G}} \dfrac{\abs{ (\F
    w)_{\xi}}^2}{\wh{K}_{\xi}}  <+\infty },\label{eq:8}
\end{equation}
where $(\F w)_{\xi}=0$ whenever $\wh{K}_{\xi}=0$. The space $\hh$ is
endowed with the scalar product 
\begin{equation}
  \label{eq:6}
  \scal{w}{w'}_\hh = \sum_{\xi\in \wh{G}}  \dfrac{(\F w)_{\xi}
  \overline{(\F w')_{\xi} }}{\wh{K}_{\xi} } .
\end{equation}
It is a standard result that $\hh$ is a  Hilbert space with a
continuous embedding $j: \hh \to L^2$. Moreover,  $\hh$ is invariant
under translations, i.e.\ $T_t\hh=\hh$  for any $t\in G$.  

\begin{rmk}\label{rmk:basis}
Recall that $\set{e_\xi}_{\xi\in \wh{G}}$ is the (Fourier basis)  of
$L^2$ and set 
\begin{equation}
  \begin{split}
    & \wh{G}_*=\set{\xi\in \wh{G},\wh{K}_{\xi}\neq 0 } \\
    &  f_\xi = \wh{K}_{\xi}^{\frac 12} e_\xi \qquad \xi\in\wh{G}_*.
  \end{split}
\label{eq:58} \ 
\end{equation}
Then $\set{f_\xi}_{\xi\in \wh{G}_* }$ is a basis of $\hh$ and,  for any $w\in\hh$
\begin{equation}
  \label{eq:71bis}
  \scal{w}{f_\xi}_\hh = \dfrac{(\F w)_\xi }{\wh{K}_\xi^{\frac{1}{2}}} .
\end{equation}
\end{rmk}

\begin{rmk}
 Assume that  $\wh{K}$ is in $\ell^1$. Denote by $k: G\times G\to \C$
\[ k(t,t')=(\F^{-1} \wh{K})(t-t'),\]
then $k$ is a positive definite kernel and $\hh$ is the reproducing kernel
Hilbert space with reproducing kernel $k$. Conversely, by Bochner's
theorem any translational invariant reproducing kernel Hilbert space on $G$ with
a continuous integrable reproducing  kernel is of the above form. 

We recall that, given an arbitrary set \( G \), a reproducing kernel Hilbert space \( \hh \) is a Hilbert space of functions from \( G \) to \( \R \), equipped with a kernel function \( k: G \times G \to \C \) such that, for all \( t \in G \), \( k(t, \cdot) \in \hh \), and for all \( w \in \hh \), \( w(t) = \scal{w}{k(t, \cdot)}_\hh \). From this definition, it follows that \( k \) is positive definite. 
It is a classic fact that the converse also holds: every positive definite kernel uniquely defines a reproducing kernel Hilbert space \citep{Schwartz1964}.

However, if the sequence $\wh{K}$ is not in $\ell^1$, in general $\hh$ is not a reproducing kernel Hilbert space on $G$, as for example, if 
$\wh{K}_{\xi}=1$ for all $\xi\in\wh{G}$, then $\hh=L^2$, which is not a reproducing kernel Hilbert space unless $G$ is finite.  
\end{rmk}

We
provide some non-trivial examples of hypothesis spaces for the one-dimensional 
torus $G=\R/\Z$, so that $\wh{G}=\Z$, and, for sake of clarity,  we
denote $\xi=\ell\in \Z$. These examples can easily extended to any dimension.
\begin{exa}[Periodic Sobolev spaces]
Fix $s>0$ and choose
          \begin{equation}
	\wh{K}_\ell= 
	\begin{cases}
	\frac{1}{4\zeta(2s)} \frac{1}{\abs{\ell}^{2s}} & \ell\neq 0 \\
	   \frac{1}{2}   & \ell=0
	\end{cases},\label{eq:7}  
      \end{equation}
where $\zeta$ is the Riemann zeta function, then $\hh$ is the Sobolev
space $H^s$ of  periodic functions on $\R$ with 
period $1$, and $\hh$ is a dense subspace of $L^2$.  If $s>1/2$, $\hh$
is a reproducing kernel Hilbert space.  With the
choice $s=1$, i.e.\
          \begin{equation}
	\wh{K}_\ell= 
	\begin{cases}
	\frac{3}{2\pi^2}  \frac{1}{\ell^2} & \ell\neq 0 \\
         \frac{1}{2}     & \ell=0
	\end{cases}, \label{eq:8bis}     
\end{equation}
we have an explicit form for the kernel given by
          \begin{equation}
	K(t)=3t^2-3t+1 \qquad t\in [0,1],\label{eq:9}
      \end{equation}
see (144.3) \cite[page. 47]{grry07}.
\end{exa}

\begin{exa}[Exponential decay on the torus]
Fix $\ga>0$ and set 
 \begin{equation}
	\wh{K}_\ell=  \frac{b-1}{b+1}  b^{-|\ell|} \qquad 
        \ell\in \Z,\label{eq:10}
      \end{equation}
where 
\[
b= (\ga+1)+\sqrt{\ga(\ga+2)}>1 \qquad \Longleftrightarrow \qquad	
\ga=\frac{(b-1)^2}{2b}, 
\]        
then $\hh$ is a reproducing kernel Hilbert space with kernel given by
          \begin{equation}
            \begin{aligned}
              K(t)  & = \dfrac{\ga}{\ga + \sin^2(\pi t)} 
            \end{aligned},
\label{eq:11}
      \end{equation}
see (147.3) \cite[page. 48]{grry07}.  Note that 
$\ga$ is an increasing function of $b$ running over
	$(0,+\infty)$ and it is strictly
	decreasing on the interval $[0,1]$ with minimum given by
	$\ga/(\ga+1)$.  Moreover, 
	\[
	\lim_{\ga\to 0} K_{0,\ga} (x)=
	\begin{cases}
	1 & x=0,1 \\
	0 &  0<x<1
	\end{cases}
	\qquad\qquad \lim_{\ga\to +\infty} K_{0,\ga} (x)= 1.
	\]
\end{exa}

\begin{exa}[Trigonometric polynomials]
Fixed $N\in \N$ and set
    \begin{equation}
  \wh{K}_\ell =
  \begin{cases}
    \wh{K}_\ell = \frac{1}{2N+1}  &  \abs{\ell}\leq N \\
     \wh{K}_\ell  = 0  & \abs{\ell}>N
  \end{cases}, \label{eq:12}
\end{equation}
then $\hh$ is a (finite dimensional) reproducing kernel Hilbert space with reproducing
kernel given by 
  \begin{equation}
K(t) = \dfrac{\sinc{(\pi(2N+1)t)}}{\sinc{(\pi t)}} \label{eq:13}.
\end{equation}
\end{exa}

\subsection{Ridge regression}\label{sec:algorithm}
We next discuss a natural ridge regression estimator adapted to learn convolution operators.
As shown next, specific expressions are available in this case.

Given an independent  family  of random variables $(X_1,Y_1)$,\ldots, $(X_n,Y_n)$,
which are  identically distributed as $(X,Y)$, we set
\begin{equation}
 \wn^\la = \argmin{w\in\hh}{ \left(\frac{1}{n} \sum_{i=1}^n \nor{ X_i\conv j(w)
    - Y_i}_2^2+\la \nor{w}^2_{\hh}\right)}  \label{eq:18}  
\end{equation}
where $\la>0$ is a positive parameter, and we denote by 
\begin{equation}
  \label{eq:32}
  \Cn^\la:\X\to \Y \qquad  \Cn^\la x= x\conv j(\wn^\la)
\end{equation}
the corresponding convolution operator. 
As usual, ${\cdot}_n$ is a
compact notation for the dependence on the training set
$\Dn=\left((X_1,Y_1),\ldots, (X_n,Y_n)\right)$.  
When $\hh$ is a translation invariant  Hilbert space, as shown by
Prop.~\ref{prop:estimator}, the estimator
$w_n^\la$ has a simple expression in the Fourier domain
\begin{equation}
  \label{eq:66}
  (\F j(\wn^\la))_\xi=
  \begin{cases}
    \dfrac{\frac{1}{n}\sum_{i=1}^n (\F Y_i)_\xi \overline{(\F
        X_i)_\xi} }{\frac{1}{n}\sum_{i=1}^n |(\F X_i)_\xi|^2+ \la
      \wh{K}_\xi^{-1} } &  \xi\in \wh{G}_* \\
  0 &  \xi\notin \wh{G}_*.
  \end{cases}\, 
\end{equation}
As expected, each Fourier component is the solution of a
one-dimensional (regularized)  least square problem where the
regularization parameter $\la \wh{K}_\xi^{-1}$ depends on the
frequency. This expression highlights that  $\wh{K}$, and hence the
corresponding hypotheses space $\hh$, allow to modulate  the regularization
parameter for each frequency $\xi\in \wh{G}$. Clearly, the above expression can be exploited for improved computations. Here, we omit this discussion and focus on the learning guarantees of the corresponding ridge regression estimator.

\section{Theoretical analysis of  the learning error}
Our main result is a quantitative analysis of the learning error in estimating  the unknown convolution operator $C_*$ using the  ridge regression  estimator $\Cn^\la$. Different error measures can be considered. A natural error measure is the expected least squares error of $\Cn$ 
\[
\EE{ \nor{\Cn^\la X-Y}_2^2 \mid \Dn} =  \EE{\nor{\Cn^\la X- C_*X}_2^2 \mid
\Dn }+ \EE{\nor{\eps}_2^2}, 
\]
where the second equality is due to eqs.\ \eqref{eq:14} and \eqref{eq:16}, and 
\begin{alignat*}{1}
  \EE{\nor{\Cn^\la X- C_*X}^2_2\mid\Dn}
\end{alignat*}
is the   excess error.
By a simple computation  and Thm.~\ref{thm:basic facts},  we can rewrite the excess error as 
\begin{equation}
  \label{eq:56}
  \EE{\nor{\Cn^\la X- C_*X}^2_2\mid\Dn} = \nor{\Sigma^{\frac 12} (\wn^\la-w_*)}_\hh^2,
\end{equation}
where $\Sigma:\hh\to\hh$ is the diagonal operator on the basis
$\set{f_\xi}_{\xi\in \wh{G}_*}$ defined by \cref{eq:58}, i.e.\
\[
\Sigma f_\xi=\sigma_\xi f_\xi\qquad \sigma_\xi =  \wh{K}_{\xi} \  \EE{
  \abs{(\F  X)_{\xi}}^2}. 
\]
see also \cref{eq:34} and \cref{eq:42}. 

\begin{rmk}\label{rmk:due}
Assume that $G=\R/\Z$. In Example~\ref{freq_loc}, we have that
\[
\sigma_\ell = \wh{K}_\ell p_\ell,
\]
whereas in Example~\ref{space_loc}
\[
\sigma_\ell = \wh{K}_\ell \operatorname{sinc}^2(2\pi \delta \ell).
\]
In both cases, $\Sigma$ is a trace-class operator for any choice of
$\wh{K}$.   
\end{rmk}

Since,  both the convolution operator $C_*$ and the ridge regression 
estimator $\Cn^\la$ are bounded operators from $\X$ to $\Y$, an
alternative error measure is 
$$\nor{ \Cn^\la - C_*}_{1,2}.$$
 By Lemma~\ref{lem:convolution_norm}, this norm can also be expressed in terms of the $L^2$ norm between the unknown convolution kernel and it ridge regression estimate, 
\begin{equation}\label{op_err}
\nor{ \Cn^\la - C_*}_{1,2} = \nor{j(\wn^\la-w_*)}_2,
\end{equation}
where we recall that $j$ is the canonical inclusion of the hypothesis
space into  the space of square integrable functions with respect to
the Haar measure of $G$.

By Assumption~\eqref{eq:17}, it is easy to see that
  \begin{alignat*}{1}
    \EE{\nor{\Cn^\la X- C_*X}^2_2\mid\Dn} & \leq
    \left(\sup_{\xi\in\wh{G}} \EE{ \abs{(\F X)_{\xi}}^2}\right)
    \nor{j(\wn^\la-w_*)}_2^2 \\ 
& \leq  \left(\sup_{\xi\in\wh{G}}\EE{
        \abs{(\F X)_{\xi}}^2}\right) 
    \left(\sup_{\xi\in\wh{G}}\wh{K}_\xi \right)
    \nor{\wn^\la-w_*}_\hh^2 \\
& \leq D_K^2 D_X^2 \nor{\wn^\la-w_*}_\hh^2.
  \end{alignat*}
The norm 
$$\nor{\wn^\la-w_*}_\hh$$ 
is yet another error measure which is more stringent than either \cref{eq:56} or \cref{op_err}.

\subsection{Main results}\label{sec:bound}
We next state our main result.  Recall that the eigenvalues of $\Sigma$
are given by 
\[\sigma_\xi= \wh{K}_{\xi} \  \EE{\abs{(\F  X)_{\xi}}^2} \qquad
  \xi\in\wh{G}\]
and set $b^{-1}=0$ if $b=+\infty$. 

\begin{thm}\label{main}
Assume that the positive part of the spectrum of $\Sigma$  is  denumerable,  i.e.\ 
\[
\sigma(\Sigma)\setminus\{0\}= \set{ \sigma_{\xi_\ell} \in (0,+\infty)
  \mid \ell\in I\subset \N}
\]
for some injective map $I\ni \ell \mapsto \xi_\ell \in \wh{G}$. 
Moreover, suppose that, for some $0\leq r\leq 1/2$,  $w_*\in\hh$ satisfies the source condition
\begin{equation}
  \label{eq:source-cond}
  \sum_{\ell\in I}  \dfrac{\abs{(\F w_*)_{\xi_\ell}}^2}{
  \wh{K}_{\xi_\ell}  \sigma_\ell^{2r}}  <
  +\infty 
\end{equation}
and, for some $b\in [1,+\infty]$,  the  family 
$(\sigma_{\xi_\ell})_{\ell\in I} $ satisfies the decay condition 
\begin{alignat}{1}
  \label{eq:69}
  \begin{cases} 
   \displaystyle{\sum_{\ell\in I} } \sigma_{\xi_\ell}    <+\infty  & b=1 \\[15pt]
\sigma_\ell  \lesssim  \frac{1}{\ell^b}  & 1<b < +\infty \\[10pt]
    \operatorname{card}(I)  <+\infty & b=+\infty.    
  \end{cases} \
\end{alignat}
Set
  \begin{equation}
\la_n = \frac{3 \kappa^2}{4} 
\begin{cases}
  n^{-\frac{1}{2r+1+b^{-1}}}  & (r,b)\neq
  (0,+\infty)\\
& \\
   \frac{\ln^2 n }{n} & r=0, b=+\infty 
\end{cases} \qquad  C_n = C_{w_n^{\la_n}} 
\label{eq:98},
\end{equation}
where $\kappa=D_XD_K$. For any $\tau>0$, there exists $n_0=n_0(\tau)$ such  that  for all
$n\geq n_0$, with probability at least $1-3 e^{-\tau}$, 
\begin{equation}
  \label{eq:70}
  \EE{\nor{\Cn^\la X- C_*X}^2_2\mid\Dn} = \nor{\Sigma^{\frac{1}{2}}
    (\wn^{\la_n}-w_*)}^2_\hh \lesssim \max\set{\tau^2,\tau}
  \begin{cases}
    n^{-\frac{2r+1}{2r+1+b^{-1}}} & (r,b)\neq (0,+\infty) \\
   \frac{\ln^2 n }{n}  & r=0, b=+\infty 
  \end{cases}
\end{equation} 
and
 \begin{equation}
\nor{\wn^{\la_n}-w_*}_\hh \lesssim 
 n^{-\frac{r}{2r+1+b^{-1}}}  \qquad  r>0.
\label{eq:error-bound}
\end{equation}
The constants in \cref{eq:70} and \cref{eq:error-bound}  depend only on  $D_X,D_K,M_\eps, \sigma_\eps, r,
b,\nor{\Sigma}_{\hh,\hh}$ and $\nor{w^*}_\hh$, and can be derived
explicitly as well as $n_0(\tau)$. 
\end{thm}
The proof of the above results together with intermediate results are given in the appendix.
We add a few observations. 
Note that if $r=0, b=+\infty$, a log factor is needed to ensure that
$\la=\la_n$ satisfies \cref{eq:28}, which is crucial to ensure that
$\nor{(\Sigma+\la \Id )^{\frac{1}{2}} (\Tn+\la \Id
  )^{-\frac{1}{2}}}_{\hh,\hh}$ is bounded, see \cref{eq:22}.  If $r=0$ there is no apriori condition on $w_*$, and one
can not have rate for $\nor{\wn^\la-w_*} _\hh$. If $b=1$, there is no
condition on the decay of the eigenvalues of $\Sigma$.  

Note that if $G$ is a second countable group, as in almost all
examples, $\sigma(\Sigma)\setminus\{0\}$ is  always denumerable.  Condition~\eqref{eq:69}  for
$b>1$ and $b=\infty$ implies that the decay condition for $b=1$, 
which is equivalent to assume that $\Sigma$ is a trace class operator.
The last condition on $\Sigma$  implies that
$\sigma(\Sigma)\setminus\{0\}$ is  always denumerable.  If the
hypothesis space $\hh$ is
finite dimensional, condition~\eqref{eq:69} always holds true with~$b=+\infty$.

 \subsection{A priori assumptions and space/frequency localization}\label{example_bound}
 
 In this section, we discuss how the regularity assumptions required to derive the learning bounds can be related to the space/frequency localization properties of the input signals.  
%\begin{exa} \label{example_bound}
Toward this end, we  specialize and illustrate the above  result in the context of  Example~\ref{freq_loc}  and
Example~\ref{space_loc} where 
\begin{itemize}
  \item[$\bullet$]  $G=\R/\Z=[0,1]$ 
  \item[$\bullet$] $\wh{K}\simeq 1/\abs{\ell}^2$ so that $\hh=H^1$
  \item[$\bullet$] $w_*\in H^2$ so that 
\[
\sum_{\ell\in\Z} \abs{(\F w_*)_\ell }^2 \ell^4  % \sum_{\ell\in\Z}
% \dfrac{\abs{(\F w_*)_\ell }^2}{ \wh{K}_\ell \sigma_\ell^{2r}}
<+\infty
\]
\item[$\bullet$] $\sigma_\ell = \wh{K}_\ell \EE{ \abs{(\F X)_\ell}^2}$
  \end{itemize}
so that the source condition~\eqref{eq:source-cond} can be written as,
\[
\sum_{\ell\in\Z} \dfrac{\abs{(\F w_*)_\ell }^2 \ell^2}{ \sigma_\ell^{2r}}
<+\infty.
\]
If $X$ is well localized in the frequency domain (as in Example~\ref{freq_loc}), then by Remark~\ref{rmk:due}
\begin{equation} \label{sigma_freq-loc}
\sigma_\ell \simeq \abs{\ell}^{-2} p_\ell  \to
\abs{\ell}^{-3} \qquad p_\ell \to \ell^{-1} .
\end{equation}
In this limit, we have that
\begin{alignat*}{1}
     &  b  = 3 \\
      & r = 1/3 \\
      & \frac{2r+1}{2r+1+b^{-1}}  =  \frac{5}{6} \\
      &  \frac{r}{2r+1+b^{-1}}  = \frac{1}{6} \,.
  \end{alignat*}
If $X$ is well localized in the space domain (as in Example~\ref{space_loc}), then by Remark~\ref{rmk:due}
\begin{equation} \label{sigma_space-loc}
\sigma_\ell \simeq \wh{K}_\ell \operatorname{sinc}^2(2\pi \delta \ell) \simeq \abs{\ell}^{-2}
    \operatorname{sinc}^2(2\pi\delta\ell)\  \operatornamewithlimits{\longrightarrow}_{\delta\to 0} \ \abs{\ell}^{-2} ,
\end{equation}
so that
\begin{alignat*}{1}
     &  b  =2 \\
      & r = 1/2 \\
      & \frac{2r+1}{2r+1+b^{-1}}  =  \frac{4}{5} \\
      &  \frac{r}{2r+1+b^{-1}}  = \frac{1}{5}.
  \end{alignat*}
Note that, if the aim is to recover the vector $w_*\in \hh$, it  is
better to select  sampling signals that are well-localized in space, as
expected. However, if  goal is to recover the convolution operator
$C_*=\cdot \conv w_*$ to  predict new outputs,  it is
better select  sampling signals that are well-localized in frequency.

\subsection{Discussion}

The bounds  in Theorem~\ref{main} provide a quantitative analysis of the learning error for ridge regression in the context of convolution operators. The results highlight the interplay between the decay properties of the eigenvalues of the covariance operator \(\Sigma\),  which describes the distribution of the input signals, and  the regularity condition on the convolution kernel \(w_*\).

The non-asymptotic bounds~\eqref{eq:70} and~\eqref{eq:error-bound} depend on two parameters: the regularity \(r\), which characterizes the smoothness of \(w_*\) through the source condition~\eqref{eq:source-cond}, and the decay rate \(b\) of the eigenvalues \(\sigma_\xi\), as described in \cref{eq:69}. When both \(r\) and \(b\) are large, faster rates are achieved, reflecting the favorable interaction between the smoothness of the target and input distribution.
 Notably, the derived rates match similar sharp bounds for  nonparametric regression  with kernel methods and linear operator  learning. The results on the  prediction norm bound~\eqref{eq:56} could be derived from results on operator learning, while the  norm bound~\eqref{eq:70} do not have a direct operator learning analogous. 
Both estimates can be derived with a unified analysis exploiting the convolution operator properties. 

The implications of these results are twofold. First, the bounds reinforce the importance of regularity assumptions in achieving efficient learning, emphasizing how they interact with sample size and the geometry of the hypothesis space. Second, the derived error rates provide a theoretical justification for practical applications of ridge regression in structured settings, including learning Green functions or blind deconvolution, as discussed in Section~\ref{sec:related-work}.

\section{Numerical simulations}\label{sec:experiments}

In this section, we provide numerical illustrations of the theoretical results in Theorem~\ref{main}. We investigate how different types of input signal localization (in space vs.\ frequency) affect the estimation of convolution operators considering different accuracy metrics. We then present an application to partial differential equations, namely the recovery of a fundamental solution of the heat equation. This example goes beyond the theoretical results we stated in this paper since the group is $\R$, but it further illustrates the potential of our approach.

\subsection{Error behavior and localization}
\begin{figure}[t]
    \centering
    % First subplot
    \begin{subfigure}[b]{0.48\textwidth}
        \centering
        \includegraphics[width=\textwidth]{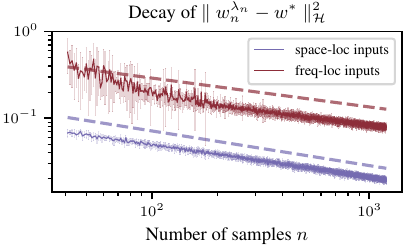}
    \end{subfigure}
    % Second subplot
    \begin{subfigure}[b]{0.48\textwidth}
        \centering
        \includegraphics[width=\textwidth]{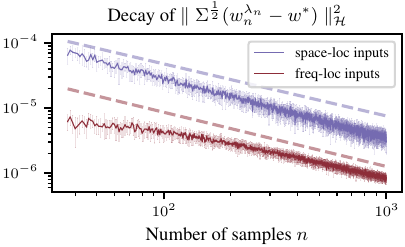}
    \end{subfigure}
    % Single caption for the entire figure
    \caption{\textbf{Error decay.} (Left)  \(\|w_n^\lambda-w_*\|^2_\hh\). (Right)   \(\|\Sigma^{\frac{1}{2}}(w_n^\lambda-w_*)\|_\hh^2\). Each curve compares frequency-localized vs.\ space-localized inputs. Dotted lines indicate the theoretical convergence rates for reference.}
   \label{fig:error}
\end{figure}

Before detailing our experimental settings, we briefly recall the main objective, namely we wish to  have an experimental evidence of the asymptotic error decay predicted by Theorem~\ref{main} under different types of input-signal localizations (in space vs.\ frequency).
\paragraph{\textbf{Setup}}
We consider the reconstruction of the convolution operator
$C_*:x \mapsto C_* x = x\conv j(w_*)$,
on the one-dimensional torus \(G=\R/\Z=[0,1]\). 
We are given i.i.d.\ samples $Y_k = C_* X_k + \varepsilon_k, k=1,\dots,n$, and aim to estimate the convolution kernel $w_*$ (and hence $C_*$) from these samples, as described in Section~\ref{sec:stat_model}, where:
\begin{itemize}
    \item The hypothesis space $\hh$ is the periodic Sobolev space $H^1$ with kernel $\wh{K}\simeq 1/\abs{\ell}^2$ (see \cref{eq:8bis,eq:9}).
    \item The target $w_\ast$ is more regular, it lies in $H^2$, implying $\sum_{\ell\in\Z} |(\F w_*)_\ell|^2\,\ell^4 <+\infty.$ 
    We further amplify its high-frequency components via a suitable factor and randomize their phases to avoid trivial cases and challenge the reconstruction.
    \item We corrupt each output $C_\ast X_k$ with additive Gaussian noise $\varepsilon_k$ of zero mean, with variance set so that the noise level is around 45\% of the signal peak. The Gaussian noise clearly satisfies \cref{eq:16}.
\end{itemize}
All functions are discretized on a grid of $2^9$ points for the FFT-based computations, following \cref{eq:66} for the ridge regression estimator $w_n^\lambda$.

\paragraph{\textbf{Input localization scenarios.}} 
As mentioned in Section~\ref{sec:stat_model}, the distribution of the input signals $X_k$ plays an important role in determining the learning rates. We thus compare two contrasting input localizations:

\begin{itemize}
    \item \textit{Frequency-localized inputs} (cf. Example~\ref{freq_loc}): The input signals are given by
    \[
    X_k(t) = e^{2\pi i \ell_k t},\quad k=1,\dots,n,
    \]
    where each frequency $\ell_k \in \mathbb{Z}$ is drawn from a power-law distribution \(p_\ell\propto |\ell|^{-\alpha}\) with $\alpha=1$.
    \item \textit{Space-localized inputs} (cf. Example~\ref{space_loc}): The input signals are given by
    \[
    X_k(t)=\frac{1}{2\delta}\,\mathbbm{1}_{\{|t-\tau_k|\le \delta\}},\quad k=1,\dots,n,
    \]
    with each \(\tau_k\) uniformly sampled on the torus, and $\delta=0.002$. Here, the distance is computed as the circular distance on the torus.
\end{itemize}
The specific choices of $\alpha$ and  $\delta$ align with the conditions in \cref{sigma_freq-loc} and \cref{sigma_space-loc}. 

\paragraph{\textbf{Implementation details.}} 
For each scenario (frequency vs.\ space localization), our numerical workflow is as follows:
\begin{enumerate}
    \item[1.] \emph{Generate data}: Draw $\{X_k\}_{k=1}^n$ according to the chosen localization distribution and compute $Y_k = C_* X_k + \varepsilon_k, k=1,\dots,n$.
    \item[2.]  \emph{Estimate $w_\ast$}: Use the ridge regression formula in the Fourier domain (cf.\ \cref{eq:66}) to compute $w_n^{\lambda_n}$. We select $\la$ via grid search over the logspace interval
$$
\lambda \in \sigma_{\max}\cdot\{10^{-3},\dots,10^{-1}\},\quad \text{where}\quad \sigma_{\max}=\max_\ell\left\{\wh{K}_\ell\frac{1}{n}\sum_{k=1}^n |(\F X_k)_\ell|^2\right\}.
$$
This heuristic selection ensures that $\la$ is scaled appropriately to the empirical covariance structure.
 We use a grid search for the optimal $\lambda$, instead of~the theoretical a a-priori choice \cref{eq:98}, to be closer to real applications where the parameters $r,b$ are unknown.
 \item[3.]\emph{Evaluate errors}:
    For each sample size $n$, we compute the errors \(\|w_n^\lambda-w_*\|^2_\hh\) and \(\|\Sigma^{\frac{1}{2}}(w_n^\lambda-w_*)\|_\hh^2\).
    \item[4.]\emph{Repeat} for increasing values of $n$ and average over multiple runs to measure sampling variability. 

\end{enumerate}

\paragraph{\textbf{Results and discussion.}}

\begin{figure}[t]
  \centering
  \includegraphics[height=5cm]{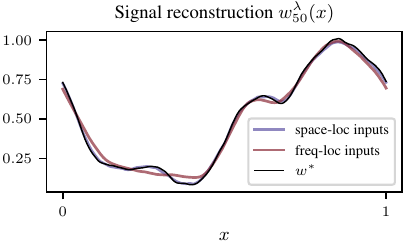}
  \caption{\textbf{Example of reconstruction.} Comparison of the true $w_\ast$  with the estimated $w_n^{\lambda_n}$ for $n=50$ in both input scenarios.}
  \label{fig:signals}
\end{figure}

Figure~\ref{fig:error} displays  the error decay for both \(\|w_n^\lambda-w_*\|^2_\hh\) (left panel) and\(\|\Sigma^{\frac{1}{2}}(w_n^\lambda-w_*)\|_\hh^2\) (right panel) as a function of the sample size $n$ in the two localization scenarios, averaged over nine independent trials per $n$. 
For reference, the theoretical convergence rates from \cref{eq:70} and \cref{eq:error-bound} (up to multiplicative constants) are also displayed. 
For space-localized inputs (right panel, Figure~\ref{fig:error}), we expanded the grid search range for the regularization parameter $\lambda$ to include smaller values (starting from $10^{-4}$). Moreover, since the magnitude of the Fourier coefficients satisfies $\abs{\wh{X}(t)_\ell} = \text{sinc}(2\pi \delta \ell)\simeq 1$ only for frequencies $|\ell|\ll \frac{1}{\delta}$, we restricted the summation in the empirical mean to this frequency range. In the same right panel, for frequency-localized inputs, we used a coarser grid consisting of $2^{12}$ points.
In Figure~\ref{fig:signals}, we illustrate a reconstruction of the target using \(n=50\) in both input scenarios.\\

As predicted by the theory (see also the discussion in Section~\ref{example_bound}), space-localized inputs lead to faster convergence in  \(\|w_n^\lambda-w_*\|_\hh\). Hence, if the primary goal is to reconstruct the convolution kernel $w_\ast$ itself, space-localized sampling is advantageous. Conversely, if the objective is to  approximate the convolution operator \(C_* = \cdot \conv w_*\) (as measured by the prediction error \(\|\Sigma^{\frac{1}{2}}(w_n^\lambda-w_*)\|_\hh^2\)), then frequency-localized inputs yield better results.

\noindent\textit{Implementation note.} The code is written in Python using FFT-based operations for fast convolution and Fourier transforms. The scripts are available at \href{https://github.com/EmiliaMagnani/learnconv}{https://github.com/EmiliaMagnani/learnconv}.

\subsection{Application to partial differential equations: heat kernel approximation} 

In this section, we illustrate how the proposed method can be used to learn a fundamental solution (Green's function) for a classic PDE, namely the heat equation. Concretely, consider the initial value problem 
\begin{equation}\label{eq:heatIVP}
\begin{split}
&\frac{\partial u}{\partial t}(x, t) - \Delta u(x, t) = 0  \quad  (x, t) \in  \mathbb{R}^D \times (0,\infty), \\
&u(x, 0) = g(x)   \quad  x \in \mathbb{R}^D .
\end{split}
\end{equation}
The function $u(x, t)$ represents the temperature at location $x\in  \mathbb{R}^D $ 
% of an infinite isolated rod at
and time $t \geq0$. 
The function $g$ specifies the temperature at $t = 0$.  
It is well known (see \cite{salsa2016partial}) that the solution  can be written as a convolution
\begin{equation} \label{eq:heat_equation}
u(x,t) = (H(\cdot,t) \ast g) (x) = \int_{\mathbb{R}^D} H(x-y,t)g(y) dy,
\end{equation}
where $H(\cdot,t)$ is the heat (Gaussian) kernel 
\begin{equation}
H(x,t) = \frac{1}{(4 \pi t)^{D/2}}  e^{-\frac{|x|^2}{4 t}}.
\end{equation}
For a fixed  $t_\ast>0$, our goal is to reconstruct the  kernel  
$$w^*(x) = H(x, t_\ast)$$
from noisy input--output data $(g_i, u_i + \varepsilon_i)_{i=1}^n$, where
each $u_i$ is obtained via \cref{eq:heatIVP} with initial condition $g_i$ and $\varepsilon_i$ models additive noise.  As already observed, in this example the group $G$ is $\R^D$, so that our theory does not apply directly.

\paragraph{\textbf{Discrete Convolution Setup.}}  

\begin{figure}[t]
  \centering
  \begin{minipage}[]{0.48\textwidth}
    \includegraphics[width=\textwidth]{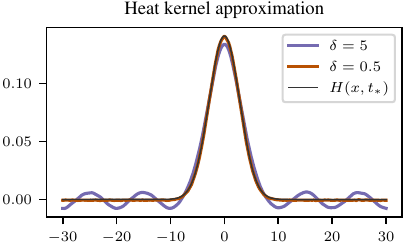}
    \caption{Heat kernel reconstruction after $n=15$ input samples for different values od $\delta$.}
    \label{fig:heat_kernel_reconstruction}
  \end{minipage}%
  \vspace{1cm}
  \begin{minipage}[]{0.48\textwidth}
    \includegraphics[width=\textwidth]{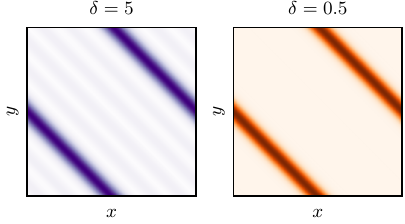}
    \caption{Corresponding convolution operators (circulant matrices).}
  \label{fig:heat_operator_reconstruction}
  \end{minipage}
\end{figure}

We restrict our experiments to a one-dimensional setting ($D=1$) and  we discretize the space variable $x$ by replacing $\R$ with $\mathbb{Z}_N$, effectively imposing boundary conditions.   This allows us to estimate the Gaussian kernel using \cref{eq:18} (cf. Example~\ref{example:circulant}).
The true heat kernel $w_\ast$
is thus represented as a vector, and the associated convolution operator $C_\ast : x \mapsto x\conv w^\ast $
is a circulant matrix (cf.\ \cref{circulant}).
In our numerical study, we choose 
\begin{itemize} 
\item A grid size $N=2^{11}$. 
\item A fixed time  $t_\ast = 3$. 
\item Input functions  $(g_i)_{i=1}^n$ given by normalized, shifted step functions—each supported on an interval of length $\delta>0$ (cf.\ \cref{space_loc})—with shifts chosen at random. 
\item Noisy outputs $u_i = g_i \conv w_\ast$ corrupted by additive noise of level $0.001$.
\end{itemize} 
In this framework the learning problem is to estimate $w^\ast$ from the data $\{(g_i,u_i+\varepsilon_i)\}_{i=1}^n$ using the ridge regression estimator in \cref{eq:66}.  For the hypothesis space we use a a reproducing kernel Hilbert space with kernel with exponential decay from \cref{eq:10} with decay $\ell=2$ and $b=1.5$. 

\paragraph{\textbf{Results and discussion}}
Figure~\ref{fig:heat_kernel_reconstruction} shows reconstructions of the heat kernel obtained from $n=15$ training samples for two values of the localization parameter  $\delta$. 
As $\delta$ decreases, the input functions $g_i$ become more “impulse-like” (approaching a discrete Dirac delta), and reveals  more direct information about $w^\ast$.
 Indeed, the reconstruction improves when $\delta$ is small. In parallel, 
Figure~\ref{fig:heat_operator_reconstruction} 
displays the estimated convolution operators (represented as circulant matrices) associated with the reconstructions. 

This experiment illustrates how our approach can be used to learn PDE solution operators. In particular, when the input functions are sharply localized in space, the recovery of the underlying convolution kernel is significantly enhanced, even with only a few training samples. The present one-dimensional setting can be generalized to higher dimensions or more complex boundary conditions by appropriately choosing the underlying group $G$ and discretization.

\section{Conclusion}  
This work provides a theoretical analysis of the problem of learning convolution operators from a finite set of input-output pairs. The case of convolution operators defined by compact Abelian groups is considered. A tailored analysis allows us to derive sharp error estimates in different norms and highlight the role of localization.

There are several avenues for future extensions of the presented study. Among others, we mention considering different regularity assumptions on the convolution kernel. Moreover,  it would be interesting to study learning of convolution operators for locally compact Abelian groups, which we plan to address in future research.

\section*{Acknowledgments}
 LR acknowledges the Center for Brains, Minds and Machines (CBMM), funded by NSF STC Award CCF-1231216. LR acknowledge the Ministry of Education, University and Research (Grant ML4IP R205T7J2KP). LR  acknowledges the European Research Council (Grant SLING 819789), the US Air Force Office of Scientific Research (FA8655-22-1-7034).  The research by EDV and LR has been supported by the MIUR Grant PRIN 202244A7YL and the MUR PNRR project PE0000013 CUP J53C22003010006 'Future Artificial Intelligence Research (FAIR)'. The research by EDV  has been supported by the MUR Excellence Department Project awarded to Dipartimento di Matematica, Universit\`a di Genova, CUP D33C23001110001. ED is a member of the Gruppo Nazionale per l’Analisi Matematica, la Probabilit\`a e le loro Applicazioni (GNAMPA) of the Istituto Nazionale di Alta Matematica (INdAM). This work represents only the view of the authors. The European Commission and the other organizations are not responsible for any use that may be made of the information it contains.
 EM and PH gratefully acknowledge financial support by the European Research Council through ERC CoG Action 101123955 ANUBIS ; the DFG Cluster of Excellence “Machine Learning - New Perspectives for Science”, EXC 2064/1, project number 390727645; the German Federal Ministry of Education and Research (BMBF) through the Tübingen AI Center (FKZ: 01IS18039A); the DFG SPP 2298 (Project HE 7114/5-1), and the Carl Zeiss Foundation, (project "Certification and Foundations of Safe Machine Learning Systems in Healthcare"), as well as funds from the Ministry of Science, Research and Arts of the State of Baden-Württemberg.   EM and PH thank the International Max Planck Research School for Intelligent Systems (IMPRS-IS) for supporting Emilia Magnani.

\newpage

\appendix

\section{Proofs}\label{proof}

\subsection{Functional tools}
We introduce a few operators useful in the proof. 

Recalling that $\set{e_\xi}_{\xi\in\wh{G}}$ is the Fourier basis of
$L^2$ and $\wh{K}\in\ell^\infty$, we denote by  $ \M : L^2\to L^2$ the
operator whose eigenvectors are $\set{e_\xi}_{\xi\in\wh{G}}$ and the
corresponding eigenvalues are $\set{\wh{K}_\xi}_{\xi\in\wh{G}}$,  i.e.\
\[
\M\, e_\xi = \wh{K}_\xi \, e_\xi\qquad \xi\in\wh{G}.
\]
Since $(\wh{K}_\xi)_{\xi\in\wh{G}}$ is a bounded positive family
  by \cref{eq:31} , clearly $\M$ is a bounded
positive operator.  We have the following result. 
\begin{lem}
With the above notations 
\begin{equation}
  \label{eq:25}
  jj^*= \M \qquad
  \nor{j}_{\hh,L^2}=\nor{j^*}_{L^2,\hh}=\nor{\M}^{1/2}_{2,2}\leq D_K \,.
\end{equation}
\end{lem}
\begin{proof}
By Remark~\ref{rmk:basis}
\[
j w = \sum_{\xi\in\wh{G}_*} \wh{K}_\xi^{\frac 12}
\scal{w}{f_\xi}_{\hh} \, e_\xi  \qquad w\in\hh,
\]
so that 
\[
j^* y = \sum_{\xi\in\wh{G}_*} \wh{K}_\xi^{\frac 12}
\scal{y}{e_\xi}_{2} \, f_\xi  \qquad y\in L^2,
\]
and
\[
j j^* y  =\sum_{\xi\in\wh{G}_*} \wh{K}_\xi \scal{w}{e_\xi}_{\hh} \,
e_\xi  = \M y .
\]
Hence, \cref{eq:25} easily follows, where the inequality is due to \cref{eq:31}.
\end{proof}
Note that $\M$ commutes with the translations $T_t$, $t\in G$, 
and any bounded positive translational invariant operator is of a such
form. 
\begin{rmk}\label{l1}
If $\wh{K}=\F K$ for some $K\in L^1$, then by convolution theorem, see \cref{eq:2},
    \begin{equation}
    \M y = K\conv y = jj^*y,\label{eq:20}
  \end{equation}
  so that $\M$ is the convolution operator by $K$.
\end{rmk}

\begin{rmk}\label{rkhs} 
  Learning convolution operator can be seen as a vector valued
  regression problem. Indeed, consider the following vector
  valued feature map
  \begin{equation}
    \Phi: \X \to \mathcal B(\hh,\Y) \qquad \Phi(x)w=x\conv
    j(w),\label{eq:95}
  \end{equation}
  which is well defined by \cref{eq:96}. Define the corresponding
  vector valued 
  reproducing kernel Hilbert space of functions from $\X$ into $L^2$
  parametrised by $\hh$, i.e.\
\[
\widetilde{\hh} = \set{f:\X\to L^2\mid f(x)=\Phi(x)w=x\conv j(w)}.
\]
Then $\widetilde{\hh}$ is, as a vector space, the  reproducing kernel Hilbert
space with vector valued reproducing kernel 
\[
\X\times \X\ni (x,x') \mapsto \Phi(x')\Phi(x)^*\in B(\Y,\Y) 
\]
and the map 
\[
\hh\in w \mapsto \Phi(\cdot)w\in \widetilde{\hh}
\]
is an isometry from $\hh$ onto $\widetilde{\hh}$ \cite{cadeto06}. In
this framework, the estimator $C_n^\la$ defined
by \cref{eq:32} is the ridge regression estimator on the hypotesis
space $\widetilde{\hh}$, whose elements $f$ are convolutions 
operators from $\X$ to $\Y$. 
\end{rmk}
In this paper, we do not explicitly use the framework of reproducing kernel Hilbert
spaces, however the map $\Phi$ plays a central role to prove our results. The following proposition recalls some basic properties of
$\Phi$. Recall that $\set{e_\xi}_{\xi\in \wh{G}}$ is the (Fourier
basis)  of 
$L^2$ and $\set{f_\xi}_{\xi\in \wh{G}_*}$, defined by \cref{eq:58}, is a base of $\hh$. 
\begin{prop}
For all $x\in L^1$
\begin{alignat}{2}
  &\Phi(x)^*: \Y \to \hh && \qquad \Phi^*(x) y = j^* (\wc{x}\conv y) \label{eq:24}\\
  & \Phi(x)^*\Phi(x):\hh\to \hh && \qquad  \Phi(x)^*\Phi(x) w=j^*(\wc{x} \conv x \conv j(w)) \label{eq:26}
\end{alignat}
Moreover,
  \begin{alignat}{2}
    \Phi(x) f_\xi & = \wh{K}_\xi^{\frac{1}{2}}\, (\F x)_\xi \, e_\xi
    &&\qquad
    \xi\in\wh{G}_*\label{eq:59} \\
    \Phi(x)^* e_\xi & = \wh{K}_\xi^{\frac{1}{2}}\, \overline{(\F
      x)_\xi} \, f_\xi &&\qquad
    \xi\in\wh{G} \label{eq:61} \\
    \Phi(x)^*\Phi(x) f_\xi & =\wh{K}_\xi\, \abs{(\F x)_\xi}^2 f_\xi
    &&\qquad \xi\in\wh{G}_* \label{eq:62}.
  \end{alignat}
Finally, the map $\Phi$ is a Lipschitz function  from $\X$ into
$B(\hh,\Y)$, i.e.\
\begin{equation}
  \label{eq:97}
 \nor{(\Phi(x) -\Phi(x')}_{\hh,L^2} \leq D_K \nor{x - x'}_1\,.
\end{equation}
\end{prop}
\begin{proof}
Fix $x\in\X$ and $y\in\Y$, for all $w\in  \hh$, by \cref{eq:33}
 \begin{alignat*}{1}
   \scal{\Phi(x)^*y}{w}_\hh&= \scal{y}{x\conv j(w)}_{L^2}=
   \scal{\wc{x}\conv y}{j(w)}_{L^2} = \scal{j^*(\wc{x}\conv y)}{w}_{\hh} ,
 \end{alignat*}
so that $\Phi(x)^*y=j^* (\wc{x}\conv y)$. Eq.~\eqref{eq:26} is a
direct consequence of the previous two equalities and the fact that
convolution is associative. 

Eq.~\eqref{eq:59} is a consequence of \cref{eq:58}
and \cref{eq:73}.  Eq.~\eqref{eq:61} is a direct consequence
of \cref{eq:59}  and both equations imply \cref{eq:62}. 

We show that $\Phi$ is a Lipschitz function. Fix $x,x'\in\X$ and $w\in \hh$,
eqs.\ \eqref{eq:1} and~\eqref{eq:25} give  
\[ 
  \nor{(\Phi(x) -\Phi(x') w}_{\hh,L^2}  =  \nor{(x-x')\conv j(w)}_{\hh,L^2} \leq  
    \nor{ x - x'}_1 \nor{j}_{\hh,L^2} \nor{w}_\hh \leq D_K\nor{ x -
      x'}_1 \nor{w}_\hh. 
\]
By taking the supremum over $w\in \hh$ with $\nor{w}_\hh\leq 1$, we
get \cref{eq:97}. 
\end{proof}

\begin{rmk}
  If $\wh{K}\in \ell^1$, so that $\hh$ is a reproducing kernel Hilbert
  space on $G$, then $\Phi(x)^*\Phi(x)$ is a trace-class operator
  since $\F x$ is in $\ell^\infty$, and, by \cref{eq:25},
  \begin{equation}\label{eq:60}
    j\Phi(x)^*: \Y \to L^2  \qquad j\Phi^*(x) y =\M (\wc{x}\conv y) . 
  \end{equation}.
\end{rmk}
\begin{rmk}
If $\wh{K}=\F K$ for some $K\in L^1$, by eqs. \eqref{eq:24} and~\eqref{eq:20}  the reproducing kernel of the hypothesis space
is given by 
\[
\Phi(x,x')y= (x'\conv K*\conv \wc{x})\conv y\qquad y\in L^2,
\]
which makes clear the relationship between the scalar reproducing
kernel Hilbert space $\hh$ and the vector valued  reproducing
kernel Hilbert space $\widetilde{\hh}$.  
\end{rmk}

\subsection{Probabilistic tools}
Associated to the feature map $\Phi$, we introduce some useful  random
variables, which play a  central role in the proofs.  By \cref{eq:97} the map
$\Phi$  is continuous from $\X$ into $\mathcal B(\hh,L^2)$, 
then $\Phi(X), \Phi(X)^*$ and $\Phi(X)^*\Phi(X)$ are random variables taking value in
$ B(\hh,\Y),  B(\Y,\hh)$ and $B(\hh)$, respectively, and $\Phi(X)^*Y$
is a random variable taking value in $\hh$. 
\begin{thm}\label{thm:basic facts}
The random variables $\Phi(X)$ and $\Phi(X)^*\Phi(X)$ are bounded by
 % & \Phi(X) \in \mathcal B(\hh,Y) \label{eq:82}\\
% &  \Phi(X)^*\Phi(X) \in \mathcal B(\hh) \label{eq:30} 
\begin{equation}
  \label{eq:83}
  \nor{\Phi(X)}^2 _{\hh,\Y} = \nor{\Phi(X)^*\Phi(X)} _{\hh,\hh} \leq
  \kappa^2 \qquad \text{almost surely}.
\end{equation}
The expectation of $\Phi(X)^*\Phi(X)$ and $\Phi(X)^*Y$ exist as 
Bochner integrals in $\mathcal B(\hh)$ and $\hh$, respectively. Set
\[ \Sigma:  \hh\to \hh \qquad \Sigma = \EE{ \Phi(X)^*\Phi(X) }, \]
then
%\[\Phi(X)^* Y \in \hh \label{eq:29}\]
\begin{alignat}{1}
  \Sigma w  & = j^* \left( \EE{\wc{X}\conv {X}} \conv j(w) \right) \label{eq:34}\\ 
   \Sigma f_\xi & = \wh{K}_\xi \EE{ \abs{(\F  X)_\xi}^2} \qquad \xi\in\wh{G}_*
   f_\xi \label{eq:42} \\
\EE{ \Phi(X)^* Y } & = \Sigma w_*\label{eq:35}
\end{alignat}
where the expectation of $\wc{X}\conv {X}$ exists as Bochner integral
in $L^1$. 
\end{thm}
\begin{proof}
By eqs. \eqref{eq:97} and~\eqref{eq:17}, we get that
\[
\nor{\Phi(X)^*\Phi(X)}_{\hh,\hh}^{\frac12}=\nor{\Phi(X)}_{\hh,L^2}
\leq D_K \nor{X}_1 \leq  D_KD_X=\kappa \qquad\text{a.s.},
\]
so that $\EE{\nor{\Phi(X)^*\Phi(X)}_{\hh,\hh}}$ is finite and the
expectation $\Sigma$ of $\Phi(X)^*\Phi(X)$ exists as Bochner
integral in $\mathcal B(\hh)$.   By a similar argument,  $\wc{X}\conv
{X}$ is bounded in $L^1$ and its expectation exists as Bochner integral
in $L^1$.  By \cref{eq:14}, $\Phi(X)^*Y=\Phi(X)^* \Phi(X) w_* +\Phi(X)^*
\eps$. Since $=\Phi(X)^* \Phi(X)$ is bounded, so is $\Phi(X)^* \Phi(X) w_*$
and $\EE{ \nor{\Phi(X)^* \Phi(X) w_*}_\hh}$ is
finite. By \cref{eq:16}  and $\nor{\Phi(X)^*}_{L^2,\hh}=
\nor{\Phi(X)}_{\hh,L^2}\leq D_X$, then
\[
\EE{\nor{\Phi(X)^*\eps}_\hh} \leq D_X \EE{\nor{\eps}_2} \leq D_X
\EE{\nor{\eps}^2_2}^{\frac 12} \leq D_X \sigma_\eps,
\] 
so that, as above, the expectation of $\Phi(X)^*Y$ exists as Bochner
integral in $\hh$. 

For all $w\in \hh$, by \cref{eq:26} it holds that
\[
\Sigma w = \EE{ \Phi(X)^*\Phi(X) w} = \EE{ j^*\left(\wc{X}\conv X \conv
  j(w)\right)}= j^*(\EE{ \wc{X}\conv X }\conv j(w) ). 
\] 
Taking into account \cref{eq:24},
\[
  \begin{aligned}
    \EE{ \Phi(X)^*Y} & =\EE{ \Phi(X)^*(\Phi(X)w_*+\eps)} = 
      \Sigma w_* + \EE{ \Phi(X)^* \EE{ \eps\mid X}} =\Sigma w_* 
  \end{aligned}
\]
by \cref{eq:16}.  Finally, since the map
\[
L^1\ni x \mapsto \F x  \in \ell^\infty
\]
is continuous and 
\[
\nor{ \F x }_\infty \leq \nor{x}_{1},
\]
then the random variable $\abs{\F X}^2$, taking value in
$\ell^\infty$, is bounded,  so that it has finite expectation. Fix $\xi\in
\wh{G}_*$, by taking the expectation of \cref{eq:62}, we
get \cref{eq:42}.  
\end{proof}
% It is useful to introduce the  following population version of ridge regression
% \begin{alignat}{1}
%   w^\la & =\argmin{w\in\hh}{\left( \EE{\nor{ X\conv j(w) - Y}^2_2} +
%       \la \nor{w}^2_\hh\right)} \label{eq:39}\\
%         & =\argmin{w\in\hh}{\left( \EE{\nor{ X\conv j(w) - X\conv j(w_*)}^2_2} +
%       \la \nor{w}^2_\hh\right)}\label{eq:72} \\
%  & =\argmin{w\in\hh}{\left(\int_{L^1}\nor{ \Phi(x) (w-w_*)}^2_{2} d\mu(x)+
%       \la \nor{w}^2_\hh \right)},
% \end{alignat}
% where the second equality is a consequence
% of~\eqref{eq:14}-\eqref{eq:16} and in the third equality 

The following result provides an explicit form for $\wn^\la$. We first
introduce some useful operators.  Let $\mu$ be the law of the random
variable $X$, which is a measure on $L^1$,  $L^2_\mu=L^2(L^1,\mu,L^2)$
and  $\nor{\cdot}_\mu$ the corresponding  norm. Denote by $\oplus_1^n
\Y$ the  direct sum  of $n$ copies of $\Y$ with the normalised norm 
\[
\nor{\oplus_i y_i}_{n }^2=\frac{1}{n}\sum_{i=1}^n \nor{y_i}_{\Y}^2.
\]  
Set 
\begin{alignat}{2}
  & \Sr: \hh \to  L^2_\mu  \qquad && (\Sr w)(x) = \Phi(x)w \qquad
  \mu\text{-a.e. }x\in\X 
  \label{eq:48}\\
 & \Sn: \hh \to  \oplus_1^n \Y \qquad && \Sn w = \oplus_1^n
  \Phi(X_i)w \label{eq:51}.
\end{alignat}
It is known \cite{caponnetto2007optimal} that
\begin{alignat}{3}
 & \Sr^*: L^2_\mu  \to \hh \qquad && \Sr^* f=  \EE{\Phi(X)^*f(X)}
 &&\quad  f\in L^2_\mu\label{eq:49}\\ %=\Sigma w_* 
  & \Sr^*\Sr:\hh\to\hh  \qquad &&\Sr^*\Sr= \Sigma\label{eq:50} \\ 
  & \Sn^*: \oplus_1^n \Y  \to \hh \qquad && \Sn^* \mathbf y =
  \frac{1}{n} \sum_{i=1}^n \Phi(X_i)^*y_i 
%= j^*( \frac{1}{n} \sum_{i=1}^n \wc{X_i}\conv W_i)
&& \quad\mathbf y =\oplus_1^n y_i\label{eq:52}\\
 &\Tn =\Sn^*\Sn:\hh\to\hh  \qquad &&\Sn^*\Sn= \frac{1}{n} \sum_{i=1}^n 
  \Phi(X_i)^*\Phi(X_i)  
%=  j^* (\frac{1}{n} \sum_{i=1}^n  \wc{X}_i\conv X_i) j 
.\label{eq:53}
\end{alignat}
Moreover, by eqs.\ \eqref{eq:59}--\eqref{eq:62}, we get
\begin{alignat}{1}
  (\Sr w)(x) & =\sum_{\xi\in\wh{G}_*}\wh{K}_\xi^{\frac 12} (\F x)_\xi\,
  \scal{w}{f_\xi}_\hh \, e_\xi  \label{eq:74} \\
  \Sr^* f & = \sum_{\xi\in\wh{G}_*}\wh{K}_\xi^{\frac 12}  
  \EE{\overline{(\F X)_\xi}\, \scal{f(X)}{e_\xi}_\hh }\, f_\xi \label{eq:76}\\
 % \Sigma w & = \sum_{\xi\in\wh{G}_*} \wh{K}_\xi  \EE{\abs{(\F
 %   X)_\xi}^2} \scal{w}{f_\xi}_\hh \, f_\xi \label{eq:77}\\
 (\Sn w)_i & =\sum_{\xi\in\wh{G}_*}\wh{K}_\xi^{\frac 12} (\F X_i)_\xi\,
  \scal{w}{f_\xi}_\hh \, e_\xi  \label{eq:78}\\
  \Sn^*\mathbf y  & = \sum_{\xi\in\wh{G}_*}\wh{K}_\xi^{\frac 12}  
  \left(\frac{1}{n} \sum_{i=1}^n \overline{(\F X_i)_\xi}\, \scal{y_i}{e_\xi}_\hh
   \right)\, f_\xi  \label{eq:79}\\
\Tn w & = \sum_{\xi\in\wh{G}_*} \wh{K}_\xi  \left(\frac{1}{n}
  \sum_{i=1}^m \abs{(\F X_i)_\xi}^2 \right) \scal{w}{f_\xi}_\hh \,
f_\xi  \label{eq:80}
\end{alignat}

\begin{rmk}
We stress that $\Sn$, $\Sn^*$ and $\Tn$ are random variables taking
values in $\mathcal B(\hh,\Y)$, $\mathcal B(\Y,\hh)$ and $\mathcal
B(\hh)$, respectively, since they depend on the training set $\Dn$. 
\end{rmk}

% By~\eqref{eq:59}--\eqref{eq:62} we get that 
% \begin{subequations}
%   \begin{alignat}{2}
%     \Phi(X)^*\Phi(X) f_\xi &=  \wh{K}_\xi\, \abs{(\F X)_\xi}^2 f_\xi
%     && \xi\in\wh{G}_*  \label{eq:63}\\
%     \Phi(X)^* Y & =\sum_{ \xi\in\wh{G}_* }  \wh{K}_\xi^{\frac{1}{2}}\,
%     \overline{(\F X)_\xi}\scal{Y}{e_\xi} \, f_\xi \label{eq:64} \\
%    \Sigma f_\xi & = \wh{K}_\xi\, \EE{\abs{(\F X)_\xi}^2}\,  f_\xi
%     && \xi\in\wh{G}_* \label{eq:65}. 
%   \end{alignat}
% \end{subequations}
% It is useful to introduce the corresponding distribution dependet
% minimisation problem
% \begin{equation}
%   \label{eq:39}
%   w^\la =\argmin {w\in\hh}{\left( \EE{\nor{ \Sr w - Y}^2_2} + \la \nor{w}^2_\hh\right)}
% \end{equation}
% so that

As a consequence of the fact that Fourier transform diagonalises the
convolution and the hypothesis space is translation invariant, we get
an explicit formula for $\wn^\la$ %and $w^\la$
 in the Fourier domain.

\begin{prop}\label{prop:estimator}
Fix $\la>0$, then 
\begin{alignat}{1}
\wn^\la  & = ( \Tn+\la \Id )^{-1} \Sn^* \mathbf Y \label{eq:38}
%\\
% w^\la & = ( \Sigma +\la \Id )^{-1} \Sigma w_*, \label{eq:40} 
\end{alignat}
where $\mathbf Y=\oplus_i Y_i\in \oplus_1^n \Y $. Moreover, 
\begin{alignat}{1}
  \label{eq:41}
 (\F j\wn^\la )_\xi & =
 \begin{cases}
 \dfrac{\frac{1}{n}\sum_{i=1}^n (\F Y_i)_\xi \overline{(\F
       X_i)_\xi}}{ \frac{1}{n}\sum_{i=1}^n |(\F X_i)_\xi|^2+ \la \wh{K}_\xi^{-1}} &  \xi\in\wh{G}_* \\
  0 &   \xi\notin\wh{G}_* 
 \end{cases} % \\
 % \label{eq:43}
 %  (\F j w^\la)_\xi & =
 %  \begin{cases}
 %    \dfrac{ \EE{ \abs{(\F X)_\xi}^2 }}{\EE{ \abs{(\F X)_\xi}^2 }+ \la \wh{K}_\xi^{-1}}
 %    (\F j w_*)_\xi & \xi\in\wh{G}_* \\
 %   0   & \xi\notin\wh{G}_* 
 %  \end{cases}
\end{alignat}
\end{prop}

\begin{proof}
By using the operator $\Sn$ the minimisation problem in \cref{eq:18}
reads as  
\begin{equation}
  \label{eq:37}
  \wn^\la =\argmin{w\in\hh}{\left( \nor{ \Sn w - \mathbf{Y}}^2_{n} + \la \nor{w}^2_\hh\right)},
\end{equation}
where $\mathbf{Y}=(Y_1,\ldots,Y_n)$. Hence, a standard result
gives \cref{eq:38}. Fix $\xi\in\wh{G}_*$, 
by eqs. \eqref{eq:79} and~\eqref{eq:80}
then
\[
\scal{\wn^\la}{f_\xi}_\hh = \dfrac{1}{\wh{K}_\xi \frac{1}{n}
  \sum_{i=1}^m \abs{(\F X_i)_\xi}^2+\la }  \wh{K}_\xi^{\frac 12}  
 \frac{1}{n} \sum_{i=1}^m \overline{(\F X_i)_\xi}\, \scal{Y_i}{e_\xi}_\hh.
\]
For $\xi\in\wh{G}_*$, \cref{eq:41} is now consequence
of \cref{eq:71bis}.  If $\xi\notin\wh{G}_*$,
\cref{eq:41} follows observing that $(\F j(w))_\xi=0$ for all
$w\in\hh$ by definition of $\hh$.  %The proof of~\eqref{eq:43} is similar. 
\end{proof}

\subsection{Decomposition error} 
The next proposition  is based on the erro decomposition  in~\cite{rucaro15}. 
\begin{prop}
For any $\la>0$
\begin{alignat}{1}
  \label{eq:71}
  \nor{\Sigma^{\frac{1}{2}} (\wn^\la-w_*)}_\hh& \leq\nor{
    (\Sigma+\la \Id )^{\frac{1}{2}} (\Tn+\la \Id )^{-1}
    (\Sigma+\la \Id )^{\frac{1}{2}}}_{\hh,\hh} \,\,\times\nonumber\\
& \quad\times \,\,  \left(\nor{  (\Sigma+\la \Id )^{-\frac{1}{2}}
    \Sn^*  \boldsymbol{\eps}}_\hh
+  \la  \nor{ (\Sigma+\la \Id )^{-\frac{1}{2}}w_* }_\hh\right),
\end{alignat}
where $\boldsymbol{\eps}=(\eps_1,\ldots,\eps_n)$ with $\eps_i=Y_i
-X_i\conv j(w_*)\sim\eps$, and
\begin{equation}
  \label{eq:99}
  \begin{split}
    \nor{\wn^\la-w_*}_\hh& \leq\nor{ (\Sigma+\la \Id )^{\frac{1}{2}}   (\Tn+\la \Id )^{-1}
   (\Sigma+\la \Id )^{\frac{1}{2}}}_{\hh,\hh} \,\,\times\\
 & \quad\times \,\, \left(\frac{1}{\sqrt{\la} }\nor{ (\Sigma+\la \Id
       )^{-\frac{1}{2}} \Sn^* \boldsymbol{\eps}}_\hh + \sqrt{\la} \nor{ (\Sigma+\la \Id )^{-\frac{1}{2}}w_* }_\hh\right).
 \end{split}
\end{equation}
\end{prop}
\begin{proof}
 By \cref{eq:38}  ,
  \begin{equation}
  \begin{split}
    \wn^\la - w_* & = (\Tn+\la \Id )^{-1} \Sn^* (\Sn
    w^*+\boldsymbol{\eps}) -w^* =  (\Tn+\la \Id )^{-1} \left((\Tn -
      (\Tn+\la \Id ))w^*+\Sn^*\boldsymbol{\eps}\right) \\ 
& =  (\Tn+\la \Id )^{-1} (\Sn^* \boldsymbol{\eps}- \la w_*).
  \end{split}
\end{equation}
Moreover, 
\begin{equation}
  \begin{split}
    (\Tn+\la \Id )^{-1}= (\Sigma+\la \Id )^{-\frac{1}{2}}
    (\Sigma+\la \Id )^{\frac{1}{2}} (\Tn+\la \Id )^{-1}
   (\Sigma+\la \Id )^{\frac{1}{2}}  (\Sigma+\la \Id )^{-\frac{1}{2}} 
  \end{split}
\end{equation}
so that 
\begin{equation}
  \begin{split}
     \nor{\Sigma^{\frac{1}{2}}(\wn^\la - w_*)}_\hh  & \leq
     \nor{\Sigma^{\frac{1}{2}} (\Sigma+\la \Id )^{-\frac{1}{2}}}_{\hh,\hh}
     \,\, 
\nor{(\Sigma+\la \Id )^{\frac{1}{2}} (\Tn+\la \Id )^{-1 }(\Sigma+\la \Id )^{\frac{1}{2}}}_{\hh,\hh} 
\,\, \times\\
& \quad\times\,\, \left( \nor{ (\Sigma+\la \Id )^{-\frac{1}{2}}\Sn^*\boldsymbol{\eps}}_\hh +\la \nor{(\Sigma+\la \Id )^{-\frac{1}{2}}w_* }_\hh\right). 
  \end{split}
\end{equation}
% Since for any bounded operator $A:\hh\to\hh$
% \[
% \nor{A}_{\hh,\hh} \nor{A^*}_{\hh,\hh} = \nor{AA^*}_{\hh,\hh} 
% \]
% then 
% \begin{equation}
%   \begin{split}
%   \nor{(\Sigma+\la \Id )^{\frac{1}{2}} (\Tn+\la \Id )^{-\frac{1}{2}}}_{\hh,\hh} 
% \nor{(\Tn+\la \Id )^{-\frac{1}{2}} (\Sigma+\la \Id )^{\frac{1}{2}}}_{\hh,\hh} & =
%    \nor{(\Sigma+\la \Id )^{\frac{1}{2}} (\Tn+\la \Id )^{-1}
%      (\Sigma+\la \Id )^{\frac{1}{2}}}_{\hh,\hh} 
%   \end{split}
% \end{equation}
% so that~
% \begin{alignat*}{1}
%   \nor{\Sigma^{\frac{1}{2}} (\wn^\la-w_*)}_\hh& \leq
%   \nor{\Sigma (\Sigma+\la \Id )^{-1}}^{\frac{1}{2}}_{\hh,\hh} \,\, \nor{
%     (\Sigma+\la \Id )^{\frac{1}{2}} (\Sn+\la \Id )^{-1}
%     (\Sigma+\la \Id )^{\frac{1}{2}}}_{\hh,\hh} \,\,\times\nonumber\\
% & \quad\times \,\,  \left(\nor{  (\Sigma+\la \Id )^{-\frac{1}{2}}
%     \Sn^* \boldsymbol{\eps}}_\hh 
% +  \la  \nor{ (\Sigma+\la \Id )^{-\frac{1}{2}}w_* }_\hh\right).
% \end{alignat*}
Eq.~\eqref{eq:71} is now consequence of the fact that $\nor{\Sigma
  (\Sigma+\la \Id )^{-1}}^{\frac{1}{2}}_{\hh,\hh}\leq 1$. The proof
of \cref{eq:99} is similar by replacing the bound $\nor{\Sigma (\Sigma+\la \Id
  )^{-1}}_{\hh,\hh}^{\frac{1}{2}}\leq 1$  with $\nor{
  (\Sigma+\la \Id )^{-1}}^{\frac{1}{2}}_{\hh,\hh}\leq 1/\sqrt{\la}$. 
\end{proof}

The following two results are given in \cite[Lemma~7.2]{rucaro15} and
\cite[Lemma 3.6]{rucaro14}. 
\begin{lem}\label{lessismore} 
Set $\Delta_n = (\Sigma+\la \Id )^{\frac{1}{2}} (\Sigma -\Tn)
(\Sigma+\la \Id )^{\frac{1}{2} }$ 
and 
\[
t_{\sup,n} = \sup\set{ t \in\sigma(\Delta_n)}.
\]
On the event 
  \begin{equation}
\Omega_{n,\la}=\left\{  t_{\sup,n} \leq \frac{1}{2}   \right\}
% \Omega_{n,\la}=\left\{\sup_{\xi\in\wh{G}_* } \left( \frac{\wh{K}_\xi }{\wh{K}_\xi \abs{(\F X)_\xi}^2+\la}\left( 
%     \EE{\abs{(\F X)_\xi}^2}- \frac{1}{n}
%   \sum_{i=1}^n \abs{(\F X_i)_\xi}^2 \right)\right)\leq
% \frac{1}{2}\right\}
\label{eq:71tris}
\end{equation}
it holds that
\begin{equation}
  \label{eq:22}
 \nor{(\Sigma+\la \Id )^{\frac{1}{2}} (\Tn+\la \Id )^{-\frac{1}{2}}}_{\hh,\hh} ^2=\nor{(\Sigma+\la \Id )^{\frac{1}{2}} (\Tn+\la \Id)^{-1}
    (\Sigma+\la \Id )^{\frac{1}{2}}}_{\hh,\hh} \leq 2
\end{equation}
\end{lem}
\begin{proof}
Observe that 
\begin{equation}
  \begin{split}
    (\Tn+\la \Id )^{-1} & = ((\Tn-\Sigma+(\Sigma+\la \Id) )^{-1} \\
& = (\Sigma+\la \Id)^{-\frac{1}{2}} \left(  (\Sigma+\la
  \Id)^{-\frac{1}{2}}  (\Tn-\Sigma)  (\Sigma+\la  \Id)^{-\frac{1}{2}}
  + \Id \right)^{-1}  (\Sigma+\la \Id)^{-\frac{1}{2}}  \\
  \end{split}
\end{equation}
so that %setting $t_{\sup,n}$ to be the  supremum of the spectrum of $\Delta_n$.
  \begin{equation}
  \begin{split}
    \nor{(\Sigma+\la \Id )^{\frac{1}{2}} (\Tn+\la \Id)^{-1}
      (\Sigma+\la \Id )^{\frac{1}{2}} }_{\hh,\hh} &  = \nor{\left( \Id -
    \Delta_n \right)^{-1}}_{\hh,\hh} = \sup_{t\in\sigma(\Delta_n)}
\frac{1}{\abs{1-t}} \\
& = \frac{1}{1- t_{\sup,n} } \leq 2
  \end{split},\label{eq:27}
\end{equation}
since on the event $\Omega_{n,\la}$, $\sigma(\Delta_n)\subset (-\infty,1/2]$. 
%  so
% that $\sigma(\Delta_n)\subset[-1/2,1/2]$ and $\abs{t_{\sup,n}}\leq
% 1/2$. 
% By~\eqref{eq:77} and ~\eqref{eq:80}, 
% \[
% \sigma(\Delta_n)= \left\{ \left(\frac{1}{(\wh{K}_\xi \abs{(\F X)_\xi}^2+\la)^{\frac12}}\left(\wh{K}_\xi 
%     \EE{\abs{(\F X)_\xi}^2}- \wh{K}_\xi \frac{1}{n}
%   \sum_{i=1}^n \abs{(\F X_i)_\xi}^2 \right)
% \frac{1}{(\wh{K}_\xi \abs{(\F X)_\xi}^2+\la)^{\frac12}}\right)\mid  \xi\in\wh{G}_*\right\},
% \]
\end{proof}
The following lemma provides a bound on $\PP{\Omega_{n,\la}}$. 
\begin{lem}\label{omega} 
 Fix $\delta>0$ and $n\geq 3$,  then 
 \begin{equation}
   \label{eq:28}
   \PP{\Omega_{n,\la}} \geq 1 - \delta \qquad  \text{\rm if }\quad   
\frac{9 \kappa^2}{n} \ln\left(\dfrac{n\tr{\Sigma}} { \delta \nor{\Sigma}_{\hh
      ,\hh}}\right) \leq \la \leq \frac{3}{4} \kappa^2.
 \end{equation}
% provided that
% \begin{equation}
%   \label{eq:5}
%   n\geq 3 \ln\left( \frac{28\tr{\Sigma}}{3\nor{\Sigma}_{\hh ,\hh}
%       \delta}\right) \qquad \text{or} \qquad \la\leq.
% \end{equation}
\end{lem}
\begin{proof}
We aim to apply Thm.~\ref{thm:tropp}.  
Define the random variable taking values in $\mathcal B(\hh)$
\[
W=(\Sigma+\la\Id)^{-1} \Sigma - (\Sigma+\la\Id)^{-\frac 12}
\Phi(X)^*\Phi(X)  (\Sigma+\la\Id)^{-\frac 12}.
\]
Since both $\Sigma$ and $\Phi(X)^*\Phi(X)$ are trace class operators,
so is $W$. Moreover, by \cref{eq:34} $\EE{W}=0$ and, since $W\leq
(\Sigma+\la\Id)^{-1} \Sigma$, then 
\[
\sigma_{\sup}(W)\leq \sigma_{\sup}( (\Sigma+\la\Id)^{-1} \Sigma)\leq
1=:M. 
%\leq \frac{\kappa^2}{\la}
\]
% , by~\eqref{eq:83}
% \[
% \nor{W}_{\hh,\hh} \leq \nor{(\sigma+\la\Id)^{-1} \Sigma}_{\hh,\hh}+
% \nor{(\Sigma+\la\Id)^{-\frac 12}}^2 _{\hh,\hh} \nor{\Phi(X)^*\Phi(X) }
% =1+ \frac{D_K^2D_X^2}{\la} \leq \frac{2\kappa^2}{\la}=:M
% \]
%taking into account that~$\la\leq 1$ and that  $R\geq 1$, see~\eqref{eq:28}. 
Moreover,
\begin{alignat*}{1}
  \EE{W^2} &= \EE{ \underbrace{ (\Sigma+\la\Id)^{-\frac 12}
\Phi(X)^*}_{A^*} \underbrace{\Phi(X)  (\Sigma+\la\Id)^{-1} \Phi(X)^*}_{B}\underbrace{\Phi(X)
(\Sigma+\la\Id)^{-\frac 12}}_{A}} - (\Sigma+\la\Id)^{-1} \Sigma^2
(\Sigma+\la\Id)^{-1} \\
& \leq \mu\operatornamewithlimits{-esssup}_{x\in\X}\nor{\Phi(x)  (\Sigma+\la\Id)^{-1} \Phi(x)^*}_{2,2}\,
\EE{(\Sigma+\la\Id)^{-\frac 12}  \Phi(X)
  \Phi(X)^*(\Sigma+\la\Id)^{-\frac 12} }  \\
& \leq  \frac{\kappa^2}{\la} (\Sigma+\la\Id)^{-1} \Sigma = S,
\end{alignat*}
where the first inequality is a consequence of the fact that for any pair
of operators $A:\hh\to L^2$ and $B:L^2\to L^2$
\[A^*B A \leq \nor{B}_{2,2}  \,A^*A\] 
and H\"older inequality,  and the second inequality is due to the fact that 
\[
  \begin{split}
     \mu\operatornamewithlimits{-esssup}_{x\in\X}\nor{\Phi(x)
      (\Sigma+\la\Id)^{-1} \Phi(x)^*}_{2,2} & \leq \frac{1}{\la}
    \mu\operatornamewithlimits{-esssup}_{x\in\X}\nor{\Phi(x)\Phi(x)^*}_{2,2}
    \leq \frac{\kappa^2}{\la}.
  \end{split}
\]
where the last inequality is a consequence of \cref{eq:83}.
%and the definition of $R$.  
Clearly
\[
\Delta_n = (\Sigma+\la\Id)^{-\frac{1}{2}}  (\Sigma-\Tn)
(\Sigma+\la  \Id)^{-\frac{1}{2}} = W - \frac{1}{n} \sum_{i=1}^n W_i,
\]
We assume that
\begin{equation}
  \label{eq:19}
  \frac{\nor{S}_{\hh,\hh}^{\frac 12} }{\sqrt{n}} + \frac{1}{3n}\leq
  \frac{1}{2},
\end{equation}
then \cref{eq:75} with $t=1/2$ gives that 
  \begin{alignat*}{1}
    \PP{\sigma_{\sup}(\Delta_n)\geq \frac{1}{2}} &\leq 4
    \frac{\tr{(\Sigma +\la \Id)^{-1}\Sigma}}{\nor{ (\Sigma +\la
        \Id)^{-1}\Sigma}_{\hh ,\hh} }\exp\left(- \dfrac{n}{8
        \kappa^2\nor{(\Sigma +\la \Id)^{-1}\Sigma}_{\hh,\hh}/\la +
        4/3}\right) \\
& \leq  4
    \dfrac{(\la+ \nor{ \Sigma}_{\hh ,\hh}) \tr{\Sigma}}{\la \nor{ \Sigma}_{\hh ,\hh}} 
\exp\left(- \dfrac{n}{8
        \kappa^2/\la +
        4/3}\right) =: \delta_n,
  \end{alignat*}

since 
% where 
% \[
% \beta=\ln\left(\frac{4 \tr{S}}{\delta \nor{S}_{\hh,\hh} }\right) \leq \ln\left(\frac{11 \kappa^2}{\delta \la}\right)=\beta'
% \]
% where the   inequality is a consequence of 
\[
  \begin{split}
    &\dfrac{\tr{S}}{\nor{S}_{\hh,\hh}}  = \dfrac{\tr{(\Sigma+\la\Id)^{-1}
        \Sigma} }{\nor{(\Sigma+\la\Id)^{-1} \Sigma}_{\hh,\hh}} \leq
    \frac{\tr{\Sigma}}{\la} \frac{ \nor{\Sigma}_{\hh,\hh}
      +\la}{\nor{\Sigma}_{\hh,\hh} } \\
  &\nor{S}_{\hh,\hh} = \frac{\kappa^2}{\la} \nor{(\Sigma+\la\Id)^{-1}
    \Sigma}_{\hh,\hh} \leq \frac{\kappa^2}{\la} \\
&\nor{ (\Sigma+\la\Id)^{-1} \Sigma}_{\hh,\hh}\leq 1 
  \end{split} \,.
\]
% If $\la>3/4 \kappa^2$, then  $8 \kappa^2/\la +
% 4/3\leq 8/3\leq 3$ and,
% since $\nor{\Sigma}_{\hh,\hh} \leq \kappa^2$,  $(\la+ \nor{
%     \Sigma}_{\hh ,\hh})/\la \leq 7/3$.  Hence, $\delta_n\leq
% \delta$ 
% provided that
% \[
% n\geq 3 \ln\left( \frac{28\tr{\Sigma}}{3\nor{\Sigma}_{\hh ,\hh} \delta}\right),
% \]
% which is condition~\eqref{eq:5}. 
Since $\la\leq 3/4 \kappa^2$, $8
\kappa^2/\la + 4/3\leq  9 \kappa^2/\la$ and, since
$\nor{\Sigma}_{\hh,\hh} \leq \kappa^2$, then
 $\la+ \nor{ \Sigma}_{\hh ,\hh} \leq 7/4  \kappa^2\leq 9/4
 \kappa^2$.  Hence
\[
\delta_n \leq   \frac{9\kappa^2\tr{\Sigma}}{\la \nor{ \Sigma}_{\hh,\hh}}
\exp(- \frac{ n}{9\la \kappa^2})
\]

Fix $\delta>0$. 
Hence, $\delta_n\leq \delta$  provided that
 \begin{equation}\label{eq:4}
\ln\left( \dfrac{9 \kappa^2\tr{\Sigma}}{\la \delta\nor{ \Sigma}_{\hh
      ,\hh}} \right) \leq \dfrac{\la n}{9 \kappa^2}.
\end{equation}
This last condition is equivalent to
\[
x\ln(x)=:\dfrac{9 \kappa^2\tr{\Sigma}}{\la \delta\nor{ \Sigma}_{\hh
      ,\hh}} \ln\left( \dfrac{9 \kappa^2\tr{\Sigma}}{\la \delta\nor{ \Sigma}_{\hh
      ,\hh}} \right)  \leq \dfrac{n\tr{\Sigma}} { \delta \nor{\Sigma}_{\hh
      ,\hh}}=:y.
\]
Since $n\geq 3$,  $\delta<1$ and $\tr{\Sigma}\geq
\nor{\Sigma}_{\hh,\hh}$,  then $y\geq e$, we solve the the inequality 
\[
x\ln(x) \leq y. 
\]
We claim that any $x\leq y/\ln(y)$ satisfies the above inequality. Indeed
Since $x\ln(x)$ is a increasing function,  
\[
x\ln(x) \leq  \frac{y}{\ln y} \ln( \frac{y}{\ln y} ) = y -
\frac{y}{\ln y}   \ln\ln y \leq y
\]
since $\ln\ln y\geq 0$.   This means that \cref{eq:4} holds true
provided that
\[
\dfrac{9 \kappa^2\tr{\Sigma}}{\la \delta\nor{ \Sigma}_{\hh
      ,\hh}} \leq \dfrac{n\tr{\Sigma}} { \delta \nor{\Sigma}_{\hh
      ,\hh}} \left( \ln\left(\dfrac{n\tr{\Sigma}} { \delta \nor{\Sigma}_{\hh
      ,\hh}}\right)\right)^{-1},
\]
which is equivalent to
\[
\la \geq \frac{9 \kappa^2}{n} \ln\left(\dfrac{n\tr{\Sigma}} { \delta \nor{\Sigma}_{\hh
      ,\hh}}\right),
\]
which is condition~\eqref{eq:28}. About condition~\eqref{eq:19},
taking into account than $n\geq 3$  it is
implied by 
\[
\frac{\kappa^2}{n\la}\leq (\frac{1}{2}-\frac{1}{9})^2\qquad\iff
\qquad\la  \geq \frac {324}{49}\frac{\kappa^2}{n\la}
\]
which always holds true since, by \cref{eq:28}
\[
\la \geq 9 \ln 3  \frac{\kappa^2}{n}
\]
and $9 \ln 3\geq \frac {324}{49}$. 
\end{proof}

The following result provides a bound on $\nor{  (\Sigma+\la \Id )^{-\frac{1}{2}}
    \Sn^*  \boldsymbol{\eps}}_\hh$, as shown in 
\cite[Proof of Thm.~4, Step 3.3]{de2005learning}.  

 \begin{prop}\label{andrea}
 Fix $\tau>0$ and $n\geq 1$, with probability greater than
 $1-2e^{-\tau}$
 \begin{equation}
   \label{eq:85}
   \nor{  (\Sigma+\la \Id )^{-\frac{1}{2}}
    \Sn^*  \boldsymbol{\eps}}_\hh \leq \left(\frac{M_\eps \kappa\tau}{\sqrt{\la} n
  } +\sqrt{ \frac{2\tau \sigma_\eps^2 \tr{ (\Sigma+\la \Id )^{-1} \Sigma}}{n}}\right).
 \end{equation}
  \end{prop}
  \begin{proof}
 Define the random variable taking value in $\hh$
\[
Z=(\Sigma+\la \Id )^{-\frac{1}{2}} \Phi(X)^*\eps 
\]   
which by \cref{eq:16} satisfies $\EE{Z}= 0$.  
% \begin{alignat*}{2}
%   \EE{Z}&= 0 \qquad &\text{by}& \quad~\eqref{eq:16} % \\
  % \nor{Z}_\hh & \leq \frac{D_\eps D_X}{\sqrt{\la}} \qquad &\text{by}&
  % \quad~\eqref{eq:17}, \eqref{eq:16}  \text{ and } \nor{ (\Sigma+\la \Id
  %   )^{-\frac{1}{2}}}_{\hh,\hh} \leq   \la^{-\frac{1}{2}} .
%\end{alignat*}
Moreover, 
\[
\nor{Z}^2_\hh = \scal{ \Phi(X) (\Sigma+\la \Id
  )^{-1}\Phi(X)^*\eps}{\eps}_\hh \leq \nor{ \Phi(X) (\Sigma+\la \Id
  )^{-1}\Phi(X)^*}_{2,2} \nor{\eps}_{L^2}^2,
\]
so that, by  the tower property of the expectation, for  any $m\geq 2$
\[
\begin{split}
  \EE{\nor{Z}^m_\hh} & \leq \EE{\nor{ \Phi(X) (\Sigma+\la \Id
  )^{-1}\Phi(X)^*}_{2,2}^{m/2} \EE{\nor{\eps}_{L^2}^m\mid X }}  \\
&  \leq   \mu\operatornamewithlimits{-esssup}_{x\in\X}\nor{\Phi(x) (\Sigma+\la \Id
  )^{-1}\Phi(x)^*}^{(m-2)/2}_{2,2}    \EE{\nor{ \Phi(X) (\Sigma+\la \Id
  )^{-1}\Phi(X)^*}_{2,2}}   \times \\
& \quad \times \frac{m!}{2} M_\eps^{m-2} \sigma_\eps^2               \\
  & \leq \left( \frac{M_\eps \kappa}{\sqrt{\la}} \right)^{m-2} \EE{\tr{\Phi(X)
      (\Sigma+\la \Id )^{-1} \Phi(X)^*}}  \frac{m!}{2} \sigma_\eps^2 \\
& = \left( \frac{M_\eps \kappa}{\sqrt{\la}} \right)^{m-2} \EE{\tr{
      (\Sigma+\la \Id )^{-1}  \Phi(X)^*\Phi(X)}} \frac{m!}{2} \sigma_\eps^2\quad ,
\end{split}
\]
where the second inequality is a consequence of H\"older inequality and
condition~\eqref{eq:16} on the noise $\eps$ and the third inequality follows by
\[
\nor{\Phi(X) (\Sigma+\la \Id
  )^{-1}\Phi(x)^*}_{2,2} \leq \frac{1}{\la}
\nor{\Phi(X)}_{\hh,L^2}^2 \leq \frac{\kappa^2}{\la}
\]
and the commutative property of  the
trace. Hence, by definition of $\Sigma$,
\[
  \begin{split}
    \EE{\nor{Z}^m_\hh} & % \leq \left( \frac{D_\eps D_X}{\sqrt{\la}}
    % \right)^{m-2} \EE{\tr{\Phi(X)^*\Phi(X) (\Sigma+\la \Id )^{-1} }
    %   \nor{\Sigma_\eps(X)}_{2,2} } \\
   \leq \left( \frac{M_\eps D_X}{\sqrt{\la}}  \right)^{m-2} 
     \tr{ (\Sigma+\la \Id )^{-1} \Sigma} \frac{m!}{2} \sigma_\eps^2\leq \frac{m!}{2}
    M^{m-2} \sigma^2
  \end{split}
\]
where $M= M_\eps \kappa/\sqrt{\la}$ and $\sigma^2= \sigma_\eps^2 \tr{ (\Sigma+\la \Id )^{-1} \Sigma} $. 
For any $i=1,\ldots,n$ set $Z_i=(\Sigma+\la \Id )^{-\frac{1}{2}}
\Phi(X_i)^*\eps_i$, then $Z_1,\ldots,Z_n$ is a i.i.d. family of random
variables distributed as $Z$ and, by \cref{eq:52}, 
\[
 (\Sigma+\la \Id )^{-\frac{1}{2}}
    \Sn^*  \boldsymbol{\eps}  = \frac{1}{n} \sum_{i=1}^n Z_i.
\]
Hence, Thm.~\ref{thm:pinelis} gives that, with probability greater than
$1-2e^{-\tau}$ 
\[
\nor{  (\Sigma+\la \Id )^{-\frac{1}{2}}
    \Sn^*  \boldsymbol{\eps}}_\hh \leq \left(\frac{M_\eps \kappa\tau}{\sqrt{\la} n
  } +\sqrt{ \frac{2\tau \sigma_\eps^2 \tr{ (\Sigma+\la \Id )^{-1} \Sigma}}{n}}\right).
\]

  \end{proof}
The following result is standard in inverse problem, see for example
\cite{groetsch84}. 
  \begin{prop}
 Let $0\leq r\leq 1/2$. Assume that $w_*=\Sigma^{r}v^*$ for some
 $v^*\in \hh$. Then
 \begin{equation}
   \label{eq:57}
   \nor{(\Sigma+\la\Id)^{-\frac 12} w_*}_\hh \leq \la^{r-1/2} \nor{v^*}_\hh.
 \end{equation}
  \end{prop}
  \begin{proof}
\[
  \begin{split}
    \nor{(\Sigma+\la\Id)^{-\frac 12} w_*}_\hh& =
    \nor{(\Sigma+\la\Id)^{-\frac 12} \Sigma^{r}v^*}_\hh \leq
    \nor{(\Sigma+\la\Id)^{-\frac 12} \Sigma^{r}}_{\hh,\hh}\,
    \nor{v^*}_\hh \\
& = \sup_{t\in\sigma(\Sigma)} \left(\frac{t^r}{(\la+t)^{\frac 12}} \right)\,\nor{v^*}_\hh  
= \la^{r-1/2} \left(\sup_{t\in\sigma(\Sigma)} \frac{(t/\la)^{2r}}{
    1+t/\la}\right)^{\frac 12} \nor{v^*}_\hh  \\ 
&\leq \la^{r-1/2} \nor{v^*}_\hh,.
\end{split}
\]
Since $2r\leq 1$, the map $\tau\mapsto \tau^{2r}$ is concave with
derivative at $\tau=1$ equal to $2r$, then 
\[
 \tau^{2r} \leq 1+ 2r (\tau-1) \leq 1 + \tau \qquad \Longrightarrow
 \qquad \left(\sup_{t\in\sigma(\Sigma)} \frac{(t/\la)^{2r}}{
    1+t/\la}\right) \leq 1 ,
\]
and \cref{eq:57} is clear. 
  \end{proof}
The following result bounds  $\tr{(\Sigma+\la\Id)^{-1} \Sigma} $, as shown in 
\cite[Prop. 3]{de2005learning}.  
  \begin{prop}\label{capacity}
 Under the decay condition~\eqref{eq:69}
 \begin{equation}
   \label{eq:89}
   \tr{\Sigma+\la\Id)^{-1} \Sigma} \lesssim
     \la^{-b^{-1}} 
 \end{equation}
where the constant in $\lesssim$ depends on $b$ and $\Sigma$. 
  \end{prop}
  \begin{proof}
If $b=1$, by H\"older inequality for the trace,
\[
\tr{ (\Sigma+\la\Id)^{-1} \Sigma} \leq \tr{\Sigma} \ 
\nor{(\Sigma+\la\Id)^{-1}}_{\hh,\hh} \leq \frac{ \tr{\Sigma}}{\la}
\lesssim \la^{-1}. 
\]
If $b=+\infty$, $\Sigma$ is a finite rank operator, let 
\[\sigma_{\min} =\min\set{ t \in \sigma(\Sigma) \mid
    t> 0}>0,\]
the smallest strictly positive eigenvalue of $\Sigma$,  then $\nor{(\Sigma+\la\Id)^{-1}}_{\hh,\hh} \leq 1/
    \sigma_{\min}$, so that, by H\"older inequality for the trace,
\[
\tr{ \Sigma+\la\Id)^{-1} \Sigma} \leq \tr{\Sigma}
\nor{\Sigma+\la\Id)^{-1}}_{\hh,\hh} \leq \frac{
  \tr{\Sigma}}{\sigma_{\min} } \lesssim \la^{-0}.
\]
If $1<b<+\infty$, denote by $\set{\sigma_\ell}_{\ell\geq 1 }$ the countable family of
strictly positive eigenvalues of $\Sigma$. By definition of trace,      
\[
  \begin{split}
    \tr{(\Sigma+\la\Id)^{-1} \Sigma} & =\sum_{\ell=1}^\infty
    \frac{\sigma_\ell}{\sigma_\ell+\la} = \sum_{\ell=1}^\infty
    \frac{1}{1 +\la/\sigma_\ell}  \lesssim \sum_{\ell=1}^\infty
    \frac{1}{1 + C \la \ell ^b} \leq \int_0^\infty  \frac{1}{1 + C \la
      x^b} \, dx \\
  & \leq (C\la)^{-1/b} \int_0^\infty  \frac{1}{1 + 
      x^b} \, dx  \lesssim \la^{-1/b}
  \end{split}
\]
where $C$ is such that $\sigma_\ell\leq C^{-1} \ell^{-b}$.  This shows \cref{eq:69}.
  \end{proof}

  \begin{proof}[Proof of Thm.~$\ref{main}$]
Fix $\tau>1$ and $n\geq 3$. Assume that $\la>0$
satisfies \cref{eq:28} with $\delta=e^{-\tau}$,
% \begin{equation} 
%   \label{eq:57bis}
%   \frac{11 R^2 }{n}  (\tau+\ln n) \leq \la\leq \min\set{\nor{\Sigma}_{\hh,\hh},1}
% \end{equation}
% where 
% \[R=\max\set{ D_K D_X,\tr{\Sigma}^{\frac{1}{2}},1}.\]
then, by Lemma~\ref{lessismore} and Lemma~\ref{omega}, with probability at
least  $1-e^{-\tau}$,
  \begin{equation}
\nor{(\Sigma+\la \Id )^{\frac{1}{2}} (\Sn+\la \Id)^{-1}
    (\Sigma+\la \Id )^{\frac{1}{2}}}_{\hh,\hh} \leq 2.\label{eq:91}
\end{equation}
Moreover, Prop.~\ref{andrea} and Prop.~\ref{capacity} give that,  with probability at
least  $1-e^{-\tau}$,
\begin{equation}
  \label{eq:90}
  \begin{split}
   \nor{  (\Sigma+\la \Id )^{-\frac{1}{2}}
    \Sn^*  \boldsymbol{\eps}}_\hh &\leq \left(\frac{M_\eps \kappa\tau}{\sqrt{\la} n
  } +\sqrt{ \frac{2\tau \sigma_\eps^2 \tr{ (\Sigma+\la \Id )^{-1} \Sigma}}{n}}\right)\\
  & \lesssim \left(\frac{\tau}{\sqrt{\la} n } +\sqrt{ \frac{ \tau }{n\la^{1/b}}}\right)  \\
  \end{split}\ .
\end{equation}

We pluggin eqs.\ \eqref{eq:91} and~\eqref{eq:90} in \cref{eq:71} taking into
account \cref{eq:57}, so that with probability greater than $1-3 e^{-\tau}$
\[
\nor{\Sigma^{\frac{1}{2}} (\wn^\la-w_*)}_\hh \lesssim \max\set{\tau,\sqrt{\tau}} \left(\frac{
       1}{\sqrt{\la} n } +\sqrt{ \frac{ 1
        }{n\la^{1/b}}} + \la ^{r+1/2} \right)
\]
Assume that $(r,b)\neq (0,+\infty)$. Set $\la=\la_n$ as in \cref{eq:98},  then taking into account 
\[
\frac{1}{\la_n n^2} \simeq \left(\frac{1}{n}
\right)^{2-\frac{1}{2r+1+b^{-1}}}\leq \left(\frac{1}{n} \right)^{\frac{2r+1}{2r+1+b^{-1}}}
\]
we get 
\[
\nor{\Sigma^{\frac{1}{2}} (\wn^\la-w_*)}^2_\hh \lesssim \max\set{\tau^2,\tau}  \left(\frac{1}{n} \right)^{\frac{2r+1}{2r+1+b^{-1}}}
\]
with probability greater than $1-3 e^{-\tau}$ provided that $n$ is
large enough so that the right inequality in \cref{eq:28} holds
true. This means that
\[
9  (\ln\left(\dfrac{n\tr{\Sigma}} {\nor{\Sigma}_{\hh
      ,\hh}}\right) +\tau )\leq n^{\frac{2r+b^{-1}}{2r+1+b^{-1}}}.
\]
Let $n_0=n_0(\tau)\geq 3$ be the smallest integer such that the above inequality
holds true, then  \cref{eq:28} holds true for any $n\geq n_0$ and
this shows bound~\eqref{eq:70} if  $(r,b)\neq (0,+\infty)$.

If  $(r,b)= (0,+\infty)$, then  
\[
\nor{\Sigma^{\frac{1}{2}} (\wn^\la-w_*)}_\hh \lesssim \max\set{\tau,\sqrt{\tau} }\left(\frac{
       1}{\sqrt{n} \ln n} +\sqrt{ \frac{ 1
        }{n } }+  \frac{\ln n}{\sqrt{ n} }\right) \lesssim
    \max\set{\tau,\sqrt{\tau} } \frac{\ln n}{\sqrt n}
\]
so that \cref{eq:70} is clear by  suitable definition of $n_0$. The proof of \cref{eq:error-bound} is similar
by using \cref{eq:99} instead of \cref{eq:71}. 
  \end{proof}
\subsection{Technical results}
The following result is a standard result of convolution. 
\begin{lem}\label{lem:convolution_norm}
Fix $y\in L^2$ and set
\[C: L^1\to L^2 \qquad C x= x\conv y, \]
then 
\[
\nor{C}_{L^1,L^2}= \nor{y}_2.
\]
\end{lem}
\begin{proof}
 For all $x\in L^1$ with $\nor{x}_1\leq 1$, Young inequality in \cref{eq:1}
 gives 
\[
\nor{Cx}_2 = \nor{x\conv y}_2  \leq \nor{x}_1 \nor{y}_2 \leq \nor{y}_2 ,
\]  
so that $\nor{C}_{L^1,L^2}\leq \nor{y}_2$. Let $(u_j)_{j\in\N}\in L^1$ be
an approximation of the identity, see \cite[Prop. 2.42]{folland}, then
  $\nor{u_j}_1=1$ for all $j\in\N$, and 
\[\lim_j u_j\conv y =\lim_j C u_j =y \qquad\text{in } L^2,\]
then 
\[ \nor{y}_2 = \lim_j \nor{ C u_j} \leq \nor{C}_{L^1,L^2},\]
which shows the converse inequality. 
\end{proof}

We recall the following concentration inequality for bounded operators. The
result is stated for matrices in \cite{tropp2012user} and it can be generalized to
separable Hilbert spaces by means of the technique in
\cite[Section 3.2]{minsker17}.   
\begin{thm}[Theorem 7.3.1 of
  \cite{tropp2012user}]\label{thm:tropp}
 Let $W_1,\ldots,W_n$ be a family of independent self-adjoint random
 operators on a separable Hilbert space $\hh$ identically distributes
 as $W$, which satisfies the following conditions
 \begin{equation}
   \label{eq:57tris}
   \begin{split}
     & \EE{W} = 0 \\
     &  \sigma_{\sup}(W)  \leq M \qquad\text{almost surely}\\
     & \EE{A^2} \leq S
   \end{split}
 \end{equation}
where $S:\hh\to\hh$ is a positive trace-class operator. Then
\begin{equation}
  \label{eq:75}
   \PP{ \sigma_{\sup}\left(\frac{1}{n} \sum_{i=1}^n W_i\right)\geq t}
  \leq \frac{4\tr{S}}{\nor{S}_{\hh,\hh}} \exp\left( -\frac{n
      t^2/2}{\nor{S}_{\hh,\hh} + M t/3}\right) \qquad \forall t\geq
  \frac{\nor{S}_{\hh,\hh}^{\frac 12} }{\sqrt{n}} + \frac{M}{3 n} 
\end{equation}
and, with probability greater than $1-\delta$,
\begin{equation}
  \label{eq:81}
   \sigma_{\sup}\left(\frac{1}{n} \sum_{i=1}^n W_i\right) \leq
   \frac{2 M\beta}{3n } + \sqrt{\frac{2\beta  \nor{S}_{\hh,\hh} }{n}
   }  \qquad  \beta=\ln\left(\frac{4 \tr{S}}{\delta \nor{S}_{\hh,\hh} }\right).
\end{equation}
\end{thm}

\begin{thm}[Theorem 8.5 of
  \cite{pine94} and \cite{pine99} ]\label{thm:pinelis}
 Let $W_1,\ldots,W_n$ be a family of independent random
 variables taking value  a separable Hilbert space $\hh$ identically distributes
 as $W$, which satisfies the following conditions
 \begin{equation}
  \label{eq:86}
   \begin{split}
     & \EE{W} = 0 \\
     & \EE{\nor{W}_\hh^m} \leq \frac{1}{2} m! M^{m-2} \sigma^2,
   \end{split}
 \end{equation}
Then
\begin{equation}
\label{eq:87}
  \PP{ \nor{\frac{1}{n} \sum_{i=1}^n W_i}_\hh \geq t}
  \leq   2 \exp\left(- \dfrac{n t^2}{\sigma^2+ M t+
      \sigma\sqrt{\sigma^2+2M t}}\right)  
\end{equation}
and, with probability greater than $1-e^{-\tau}$
\begin{equation}
  \label{eq:88}
  \nor{\frac{1}{n} \sum_{i=1}^n W_i}_\hh  \leq \frac{M\tau}{n}
  +\sqrt{\frac{2\sigma^2\tau}{n}} . 
\end{equation}
\end{thm}

%\newpage
\bibliographystyle{plain}
\bibliography{biblio}

\end{document}